\documentclass[10pt,twocolumn,letterpaper]{article}

\usepackage{cvpr}              
\definecolor{cvprblue}{rgb}{0.21,0.49,0.74}
\usepackage[pagebackref,breaklinks,colorlinks,allcolors=cvprblue]{hyperref}

\definecolor{org}{HTML}{F8A145}
\usepackage{hyperref}
\hypersetup{
  colorlinks=true,
  linkcolor=blue,
  urlcolor=org,
  citecolor=org,
  linktoc=all
}

\usepackage{amsthm}
\theoremstyle{plain}
\newtheorem{thm}{Theorem}
\newtheorem{lem}{Lemma}

\theoremstyle{definition}
\newtheorem{rem}{Remark}
\newtheorem{defn}{Definition}
\newtheorem{assmp}{Assumption}

\newtheorem*{rthm1}{\textbf{Restate of Theorem 1}}
\newtheorem*{rthm2}{\textbf{Restate of Theorem 2}}

\usepackage{multirow}
\usepackage{colortbl}
\usepackage{xcolor}
\definecolor{DisColor}{HTML}{6D339C}
\definecolor{PerColor}{HTML}{285D9B}

\usepackage{url}

\usepackage{tabularx}
\usepackage{pifont}
\usepackage{titletoc}

\def \x {\mathbf{x}}
\def \z {\mathbf{z}}

\usepackage{tcolorbox}
\tcbuselibrary{skins}

\usepackage{fontawesome5} 

\usepackage{algorithmic}
\usepackage[ruled,vlined]{algorithm2e}

\def\paperID{35220} 
\def\confName{CVPR}
\def\confYear{2026}

\title{Guiding Diffusion-based Reconstruction with Contrastive Signals \\for Balanced Visual Representation}

\author{
	Boyu Han\textsuperscript{1,2}\hspace{1em} Qianqian Xu\textsuperscript{1,3,}\thanks{Corresponding authors.}\hspace{1em} Shilong Bao\textsuperscript{2}\hspace{1em} Zhiyong Yang\textsuperscript{2} \\Ruochen Cui\textsuperscript{4,5} \hspace{1em}  Xilin Zhao\textsuperscript{6} \hspace{1em} Qingming Huang\textsuperscript{2,1,*} \\
	{\textsuperscript{1} State Key Laboratory of AI Safety, Institute of Computing Technology, CAS} \\
	{\textsuperscript{2} School of Computer Science and Tech., University of Chinese Academy of Sciences} \\
    {\textsuperscript{3} Beijing Academy of Artificial Intelligence} \\
    {\textsuperscript{4} Institute of Information Engineering, CAS} \\
    {\textsuperscript{5} School of Cyber Security, University of Chinese Academy of Sciences} \\
    {\textsuperscript{6} School of Computer Science and Technology, Beijing Institute of Technology} \\
	{\tt\small \{hanboyu23z, xuqianqian\}@ict.ac.cn \hspace{1em} \{baoshilong,yangzhiyong21,qmhuang\}@ucas.ac.cn } \\ 
    {\tt\small cuiruochen25@mails.ucas.ac.cn \hspace{1em} 1120230539@bit.edu.cn}
}

\begin{document}
\maketitle
\begin{abstract}
The limited understanding capacity of the visual encoder in Contrastive Language-Image Pre-training (CLIP) has become a key bottleneck for downstream performance. This capacity includes both Discriminative Ability (D-Ability), which reflects class separability, and Detail Perceptual Ability (P-Ability), which focuses on fine-grained visual cues. Recent solutions use diffusion models to enhance representations by conditioning image reconstruction on CLIP visual tokens. We argue that such paradigms may \textbf{compromise} D-Ability and therefore \textbf{fail to} effectively address CLIP's representation limitations. To address this, we integrate contrastive signals into diffusion-based reconstruction to pursue more comprehensive visual representations. We begin with a straightforward design that augments the diffusion process with contrastive learning on input images. However, empirical results show that the naive combination suffers from gradient conflict and yields suboptimal performance. To balance the optimization, we introduce the \textbf{Diffusion Contrastive Reconstruction (DCR)}, which unifies the learning objective. The key idea is to inject contrastive signals derived from each reconstructed image, rather than from the original input, into the diffusion process. Our theoretical analysis shows that the DCR loss can jointly optimize D-Ability and P-Ability. Extensive experiments across various benchmarks and multi-modal large language models validate the effectiveness of our method. The code is available at \href{https://github.com/boyuh/DCR}{https://github.com/boyuh/DCR}.
\end{abstract}
\section{Introduction}
\label{sec:introduction}

\begin{figure}[t]
  \centering
   \includegraphics[width=\linewidth]{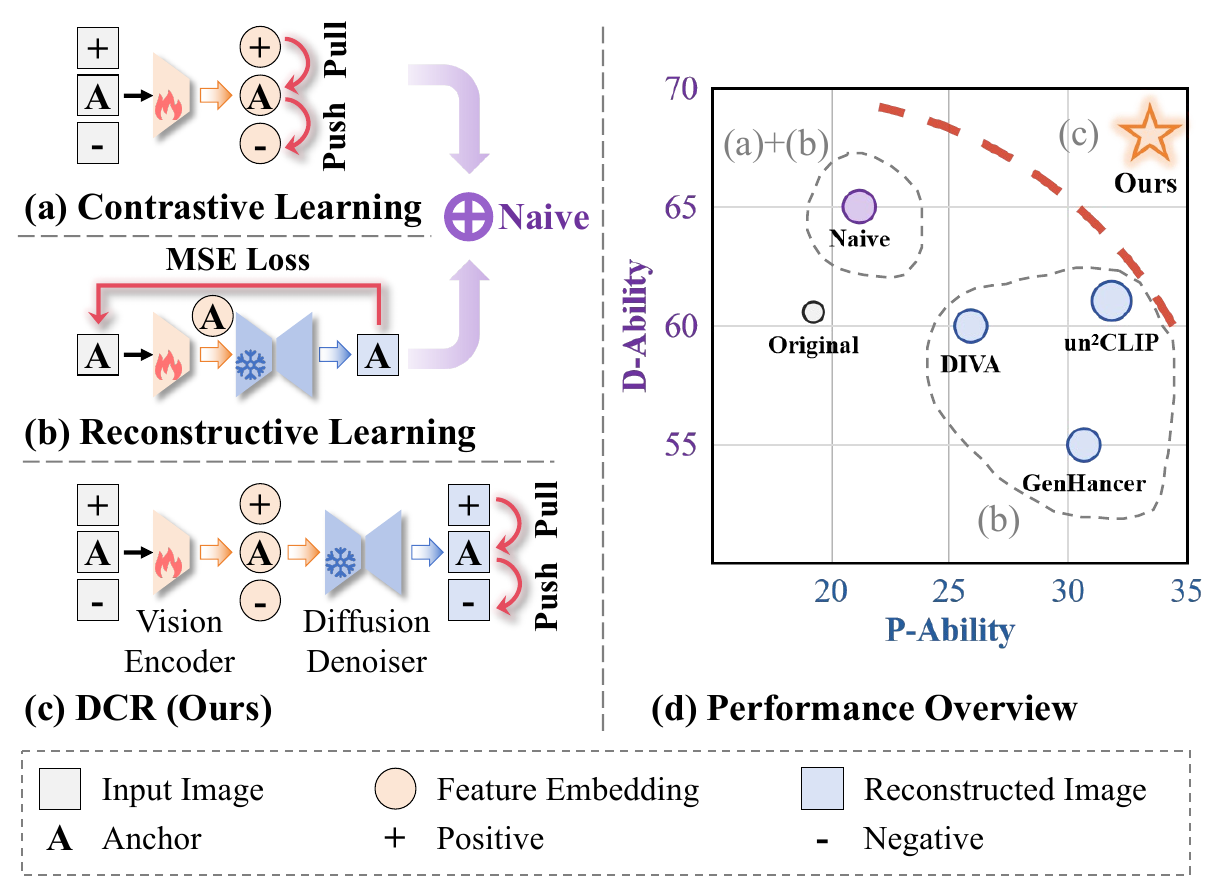}
   \caption{(a) Contrastive learning for \textbf{\textcolor{DisColor}{D-Ability}}. (b) Reconstructive learning for \textbf{\textcolor{PerColor}{P-Ability}}. (c) Our Diffusion Contrastive Reconstruction (DCR) for harmonizing \textbf{\textcolor{DisColor}{D-Ability}} and \textbf{\textcolor{PerColor}{P-Ability}}. (d) Performance overview of DCR and other methods on the OpenAI CLIP ViT-L@224 backbone.}
   \vspace{-10pt}
   \label{fig: intro overview}
\end{figure}

Contrastive Language-Image Pre-training (CLIP)~\cite{radford2021learning} is widely used as a frozen module to provide visual representations for various downstream tasks~\cite{yao2024tcp,zhou2022conditional,zhang2024overcoming}. However, its self-supervised text-image alignment is relatively coarse, which limits the model's understanding capability and has become a major bottleneck to broader deployment~\cite{kim2024extending,wangvitrix}.

The CLIP's understanding capacity has two complementary aspects~\cite{kim2021hybrid,nozawa2022evaluation}. The first is \textbf{\textcolor{DisColor}{Discriminative Ability (D-Ability)}}, which requires the model to separate categories with clear boundaries, keeping samples from the same class close while pushing different classes apart~\cite{li2014fisher,radford2021learning}. It is crucial for recognition~\cite{yan2023clip,wen2025partial,lu2025bidirectional,wang2023unified,zhang2024transfr,han2025codts}, retrieval~\cite{luo2023lexlip,zhang2024overcoming,liu2024not}, and open-vocabulary transfer~\cite{zhou2022learning,yao2024tcp,liu2025reliable,wang2022openauc,huaopenworldauc}. The second is \textbf{\textcolor{PerColor}{Detail Perceptual Ability (P-Ability)}}, which preserves information about color, direction, quantity, and structure~\cite{tong2024eyes}. It supports multimodal question answering~\cite{hudson2019gqa}, instruction following~\cite{liu2024improved}, and vision-centric reasoning~\cite{tong2024cambrian,li2024naturalbench}.

Traditional methods fine-tune the vision encoder to learn better representations under CLIP feature supervision. As shown in \cref{fig: intro overview}(a), they often rely on various forms of contrastive learning~\cite{zhou2022learning,wei2023improving,daiexploring,liu2025future}. With the rise of diffusion models~\cite{zhang2024long,lt2026hong,zhao2025bias}, recent studies turn to generative feedback for stronger enhancement~\cite{wangreconstructive,wangdiffusion,ma2025genhancer,li20252}. As shown in \cref{fig: intro overview}(b), the diffusion model takes CLIP visual features as conditions and reconstructs the image. When the reconstructed image matches the original, the visual representation has captured all the information present in the input. This process builds the reconstruction loss $\mathcal{L}_{\text{rec}}$ through a \textit{Mean-Square Error (MSE)} objective, thereby improving \textbf{\textcolor{PerColor}{P-Ability}}. However, due to the lack of class supervision, it brings little gain or may even harm \textbf{\textcolor{DisColor}{D-Ability}}, as shown in \cref{fig: intro overview}(d). A natural question arises: \textit{Can we achieve a representation that balances discriminative ability and detailed perceptual ability?}

Motivated by the fact that contrastive learning increases class separability, we propose a straightforward variant that adds a contrastive loss $\mathcal{L}_{\text{con}}$ on top of the reconstruction objective to strengthen \textbf{\textcolor{DisColor}{D-Ability}}. However, this naive design does not achieve a good balance in practice, as shown in \cref{fig: intro overview}(d). The challenge lies in the fact that the two objectives focus on different aspects: one targets CLIP feature separability, while the other relies on image-level reconstruction consistency. This mismatch causes the easier objective to dominate the other. We examine the behavior of the contrastive loss $\mathcal{L}_{\text{con}}$ and the reconstruction loss $\mathcal{L}_{\text{rec}}$ during training and find that a gradient conflict indeed exists between them. \cref{fig: Gradient Conflict}(c) shows that their gradient directions often diverge, and the conflict grows over time. As a result, $\mathcal{L}_{\text{con}}$ gradually dominates while $\mathcal{L}_{\text{rec}}$ stalls, which leads to unstable convergence, as illustrated in \cref{fig: Gradient Conflict}(a)(b). A more detailed analysis is provided in \cref{subsec: A Naive Optimization Method}.

We address this challenge with a new framework called \textit{Diffusion Contrastive Reconstruction (DCR)}. \textbf{The key improvement is to replace image-level consistency with contrastive supervision on the reconstructed images}, as shown in \cref{fig: intro overview}(c). Specifically, for each image, the anchor is the reconstruction from its own visual features. The positive is the reconstruction conditioned on its augmented view, and the negatives are reconstructions conditioned on features from other images in the mini-batch. A contrastive loss is then applied to these samples, with the diffusion model remaining frozen. This design leads to a single optimization objective and naturally avoids gradient conflicts.

We further provide a theoretical analysis under a mild assumption and show that our method can satisfy contrastive constraints (Thm. \ref{thm: D-Ability short theorem}) and reconstruction consistency (Thm. \ref{thm: P-Ability short theorem}) at the same time. Finally, comprehensive empirical studies consistently speak to the efficacy of our method.

Our main contributions are summarized as follows:
\begin{itemize}[leftmargin=*]
    \item This paper rethinks diffusion-based reconstruction methods and reveals that they improve the \textbf{\textcolor{PerColor}{P-Ability}} of visual representations, while \textbf{\textcolor{DisColor}{D-Ability}} shows little gain or may even degrade.
    \item We find that directly adding contrastive learning to supplement \textbf{\textcolor{DisColor}{D-Ability}} leads to gradient conflicts. Therefore, this paper proposes the DCR loss, which uses a single objective to avoid gradient conflict, achieving a balance between \textbf{\textcolor{DisColor}{D-Ability}} and \textbf{\textcolor{PerColor}{P-Ability}}.
    \item Our theoretical analysis and extensive experiments across $6$ CLIP backbones and diverse vision benchmarks demonstrate the effectiveness of our method in enhancing visual representations.
\end{itemize}

\begin{figure}
  \centering
  \begin{subfigure}{0.48\linewidth}
    \includegraphics[width=\linewidth]{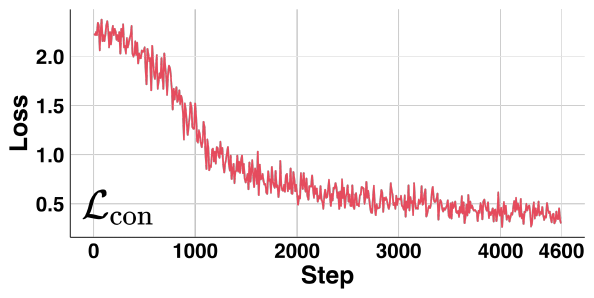}
    \caption{Contrastive loss $\mathcal{L}_{\text{con}}$.}
    \label{fig: loss_con}
  \end{subfigure}
  \hfill
  \begin{subfigure}{0.48\linewidth}
    \includegraphics[width=\linewidth]{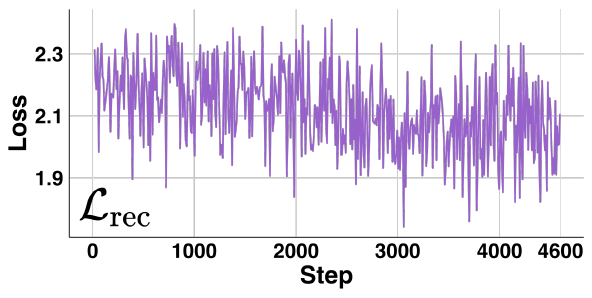}
    \caption{Reconstruction loss $\mathcal{L}_{\text{rec}}$.}
    \label{fig: loss_rec}
  \end{subfigure}
  \hfill
  \begin{subfigure}{\linewidth}
    \includegraphics[width=\linewidth]{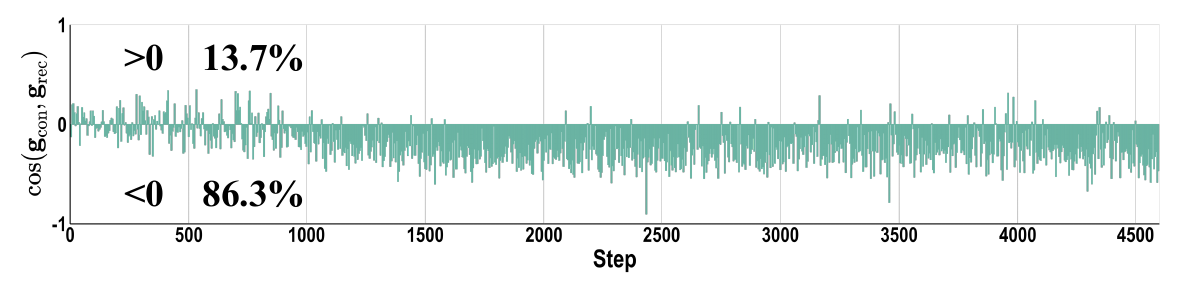}
    \caption{Cosine similarity between gradient $\mathbf{g}_{\text{con}}$ and $\mathbf{g}_{\text{rec}}$.}
    \label{fig: cos_rec_con}
  \end{subfigure}
  \caption{Training results of linearly combining contrastive learning and reconstructive learning (Naive Method) on the OpenAI CLIP ViT-L@224 backbone. The contrastive loss $\mathcal{L}_{\text{con}}$ dominates the gradients, and most training steps exhibit gradient conflicts.}
  \vspace{-15pt}
  \label{fig: Gradient Conflict}
\end{figure}
\section{Related Work}
\label{sec: relatedwork}

\noindent \textbf{Vision Encoders.} Vision encoders serve as the visual backbone of MLLMs~\cite{liu2024improved,liu2023visual,hua2024reconboost,weimoka}, transforming raw images into high-level feature representations. Among various architectures~\cite{radford2021learning,caron2021emerging,oquab2024dinov2,simeoni2025dinov3}, CLIP~\cite{radford2021learning} is one of the most widely adopted models. It also provides transferable representations for a wide range of tasks, including detection~\cite{subramanian2022reclip,mu2022slip,zhong2022regionclip,kim2023region,bao2025towards,yang2024harnessing,yang2026dirmixe,yang2023auc-oriented,han2025dual,luo2025long}, segmentation~\cite{zhou2022extract,xu2022simple,xie2022clims,liang2023open,han2024aucseg}, video understanding~\cite{xu2021videoclip,tang2021clip4caption,lin2022frozen,bose2023movieclip}, and generation~\cite{rombach2022high,nichol2022glide,podellsdxl,esser2024scaling,han2025lightfair,li2025hybrid,lione,qinmixbridge,gao2025devil,gao2025rapo++,leng2023dynamic,leng2024hypersdfusion}. Recent studies focus on enhancing CLIP's understanding capability to improve downstream performance. Some primarily strengthen its discriminative ability~\cite{zhou2022learning,wei2023improving,gong2025boosting}, while others improve its detail perceptual ability~\cite{tong2024eyes,wangreconstructive,wangdiffusion,ma2025genhancer,li20252}. However, an ideal visual encoder should achieve a balance between both aspects. To bridge this gap, this paper aims to develop an enhanced CLIP vision encoder with balanced understanding abilities, providing a stronger foundation for MLLMs.

\vskip 0.5 ex
\noindent \textbf{Generative Models for Representation Learning.} Generative modeling has recently become a powerful paradigm for improving visual representations. Early studies~\cite{shipard2023diversity,li2024semantic,tian2023stablerep} use generative models as data augmenters to expand training data and boost generalization. More recent works~\cite{wangreconstructive,wangdiffusion,ma2025genhancer,li20252} add self-supervised reconstruction objectives to diffusion frameworks, enabling models to learn richer spatial details and fine-grained semantics. Among them, DIVA~\cite{wangdiffusion} introduces diffusion-based visual feedback conditioned on CLIP embeddings, which refines internal representations through generative reconstruction. GenHancer~\cite{ma2025genhancer} extends this reconstruction process to discrete latent spaces, while un$^2$CLIP~\cite{li20252} resolves the mismatch between CLIP's pretrained feature space and generative model latents. However, these methods often overemphasize fine-grained reconstruction while neglecting discriminative learning, leading to suboptimal representations. Moreover, GenHancer and un$^2$CLIP abandon existing pretrained generative models and retrain specialized ones from scratch, incurring substantial computational costs. In contrast, our work simultaneously strengthens discriminative ability and detail perceptual ability, producing a more effective CLIP representation. Furthermore, we directly leverage existing pretrained generative models as the foundation, dramatically reducing training overhead.

\vskip 0.5 ex
\noindent \textbf{Vision-understanding Benchmarks.} The benchmarks can be broadly divided into two camps. The first focuses on assessing discriminative ability. It is typically evaluated through zero-shot clustering tasks~\cite{wen2023graph,deng2023projective} on diverse datasets such as ImageNet~\cite{deng2009imagenet}, Caltech-101~\cite{fei2004learning}, and DTD~\cite{cimpoi2014describing}. These benchmarks measure a model's generalization across domains and assess how well visual encoders preserve transferable recognition ability without task-specific fine-tuning. The second group emphasizes detail perceptual ability. MMVP~\cite{tong2024eyes} introduces a challenging benchmark with nine visual patterns, CV-Bench~\cite{tong2024cambrian} extends evaluation to spatial relations, counting, and 3D understanding, and NaturalBench~\cite{li2024naturalbench} adds natural adversarial examples to assess perceptual robustness. Together, these benchmarks provide a comprehensive evaluation of the visual representation's overall understanding ability.

\begin{figure*}[t]
  \centering
   \includegraphics[width=\linewidth]{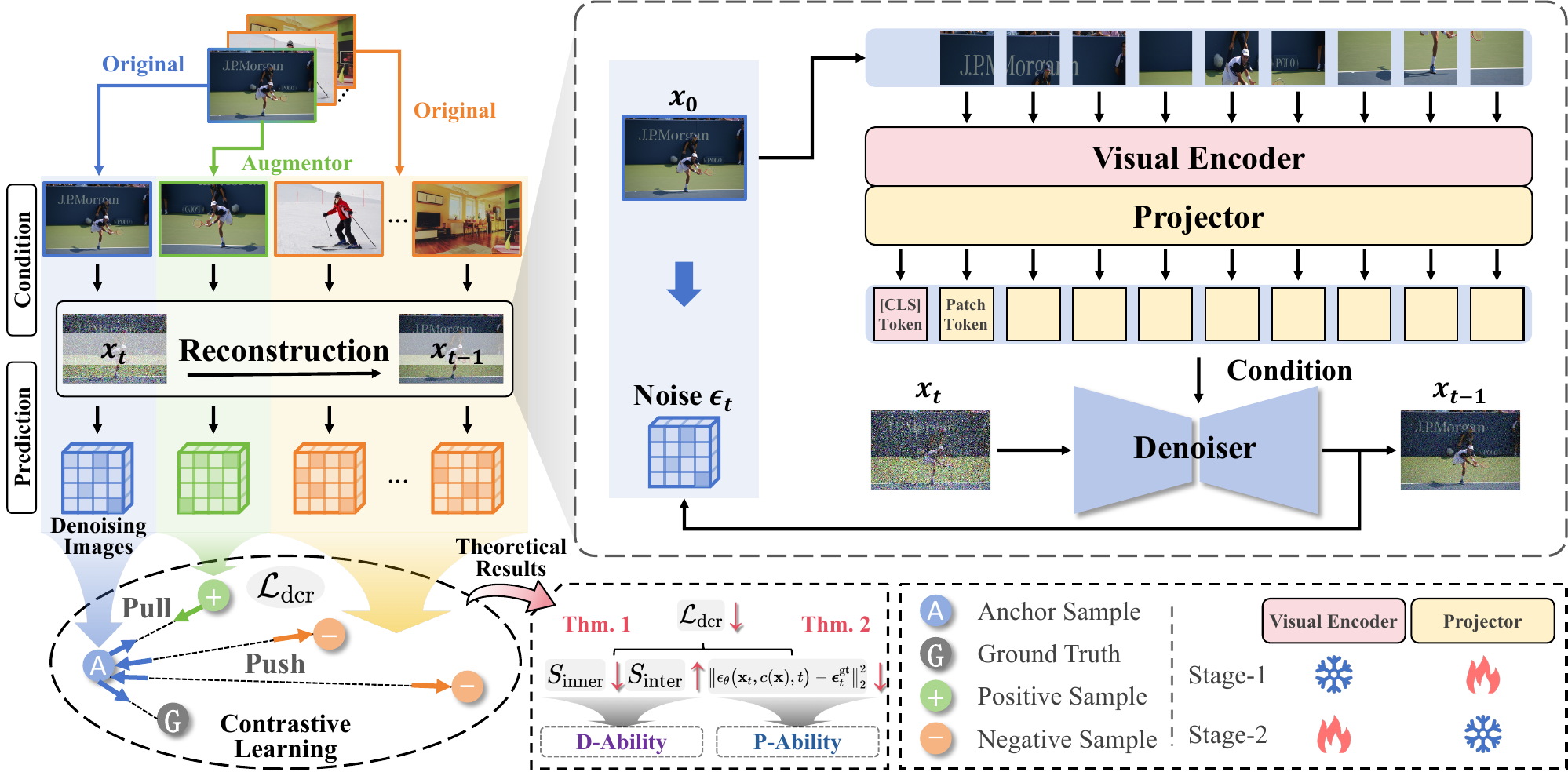}
   \caption{\textbf{An overview of Diffusion Contrastive Reconstruction (DCR).} An image is encoded by CLIP and projected into the diffusion condition space. Predicted noises from original, augmented, and negative samples form a contrastive triplet in the reconstruction image space. Training proceeds in two stages: projector alignment and encoder enhancement.}
   \vspace{-10pt}
   \label{fig: overview}
\end{figure*}
\section{Preliminaries}
\label{sec: preliminaries}

\noindent \textbf{Notation.} Let $\epsilon_{\theta}$ denote the generative model, $f_{\phi}$ denote the CLIP visual encoder, and $h_{\omega}$ denote the projection module that maps CLIP features into the generative space. The training dataset is defined as $\mathcal{D}=\{\x^i\in\mathbb{R}^{H\times W\times 3}\}_{i=1}^{n}$, where $H$ and $W$ denote the image height and width. A \textit{variational autoencoder (VAE)} encodes each image $\x$ into a latent $\tilde{\x}$. We denote the visual feature as $\z = f_{\phi}(\x)$, the condition as $c = h_{\omega}(\z)$, and the predicted noise as $\hat{\boldsymbol\epsilon}_{\theta} = \epsilon_{\theta}(\tilde{\x}_t, c, t)$ at diffusion step $t$, with the ground-truth noise $\boldsymbol\epsilon_t^{\text{gt}}$. For clarity, we include a table of symbol definitions in the Appendix.

\subsection{Diffusion Models}
\label{subsec: Diffusion Models}
Diffusion models~\cite{ho2020denoising} generate data by learning to reverse a gradual noising process. In the forward process, clean samples are progressively corrupted with Gaussian noise over $T$ steps until they become nearly random noise. The reverse process trains a denoising network to predict and remove noise step by step, reconstructing data from noise. The denoising process is formulated as
\begin{equation}
    p_\theta(\x_{t-1}|\x_t)=\mathcal{N}(\x_{t-1};\mu_\theta(\x_t,t),\sigma_t^2\mathbf{I}),
\end{equation}
where $\mu_{\theta}(\x_t, t) = \frac{1}{\sqrt{\alpha_t}}(\x_t-\frac{\beta_t}{\sqrt{1-\bar{\alpha}_t}}\epsilon_{\theta}(\x_t,t))$ is parameterized by the noise prediction network $\epsilon_{\theta}(\x_t,t)$, $\beta_t = 1 - \alpha_t$, and $\sigma_t^2$ is typically chosen as either $\sigma_t^2 = \beta_t$ or $\sigma_t^2 = \frac{1 - \bar{\alpha}_{t-1}}{1 - \bar{\alpha}_t} \beta_t$.

Stable Diffusion~\cite{rombach2022high} extends this framework to conditional generation by adding additional information $c$ into the denoising process. The conditional reverse step is expressed as
\begin{equation}
    p_\theta(\x_{t-1}|\x_t, c)=\mathcal{N}(\x_{t-1};\mu_\theta(\x_t,c,t),\sigma_t^2\mathbf{I}).
\end{equation}
By operating in a latent space through an autoencoder, Stable Diffusion achieves efficient and high-quality text-to-image synthesis, establishing diffusion models as a foundation of modern generative modeling.

\subsection{The Task of Visual Representation Learning}
\label{subsec: Visual Representation Learning}
Visual representation learning aims to learn an encoder $f_{\phi}:\mathcal{X}\!\to\!\mathcal{Z}$ that maps an image $\x$ into a visual representation $\z$. The goal is to obtain a representation with rich understanding capacity, typically characterized by two key properties: strong \textbf{\textcolor{DisColor}{Discriminative Ability (D-Ability)}} and \textbf{\textcolor{PerColor}{Detail Perceptual Ability (P-Ability)}}~\cite{kim2021hybrid,nozawa2022evaluation}.

\vskip 0.5 ex
\textbf{\textcolor{DisColor}{Property 1 (D-Ability):}} Discriminative ability describes a model's ability to accurately and robustly distinguish between different visual categories. A representation with strong discriminative power can effectively separate distinct classes in the feature space, keeping samples from the same class tightly clustered while pushing samples from different classes farther apart~\cite{li2014fisher,radford2021learning}. Mathematically, for a mini-batch with class set $\mathcal Y$ and feature set $\mathcal Z_y = \{\z_i | y_i = y\}$, we define the class center as $\mu_y = \frac{1}{|\mathcal Z_y|} \sum_{\z \in \mathcal Z_y} \z$. It is formulated as
\begin{equation}
\min S_{\text{inner}},\quad \max S_{\text{inter}},
\label{eq: minmax_s}
\end{equation}
\begin{equation}
S_{\text{inner}}=\frac{1}{|\mathcal Y|}\sum_{y\in\mathcal Y}\frac{1}{|\mathcal Z_y|}\sum_{\z\in\mathcal Z_y}|\z-\mu_y|_2^2,
\label{eq: s_inner}
\end{equation}
\begin{equation}
S_{\text{inter}}=\frac{1}{|\mathcal Y|(|\mathcal Y|-1)}\sum_{y,y'\in\mathcal Y,y\neq y'}|\mu_y-\mu_{y'}|_2^2.
\label{eq: s_inter}
\end{equation}

\textbf{\textcolor{PerColor}{Property 2 (P-Ability):}} Detail perceptual ability reflects the model's capacity to capture fine-grained visual details, enabling a comprehensive understanding of object attributes, spatial relations, and subtle semantic cues within an image. Following previous studies~\cite{wangreconstructive,wangdiffusion,ma2025genhancer,li20252}, this perception is usually evaluated through reconstruction consistency. When a diffusion model can accurately reconstruct the original image from its learned representation, the representation is considered sufficiently rich. It is formulated as
\begin{equation}
\min \mathbb{E}_{t}\Big[\|\epsilon_{\theta}(\x_t,\ h_{\omega}(f_{\phi}(\x)),t)-\boldsymbol\epsilon_t^{\text{gt}}\|_2^2\Big],
\label{eq: min_rec}
\end{equation}
where $t$ denotes the diffusion timestep, $f_{\phi}$ denotes the CLIP visual encoder, and $h_{\omega}$ denotes the projector.
\section{Methodology}
\label{sec: methodology}
In this section, we first propose a naive optimization method that aims to pursue the two properties in \cref{subsec: Visual Representation Learning}. Due to gradient conflicts, we further introduce a joint optimization approach called \textit{Diffusion Contrastive Reconstruction (DCR)}. A brief overview is provided in \cref{fig: overview}. DCR enhances CLIP's visual representations through a pretrained conditional diffusion model with contrastive signals. We also provide a theoretical analysis showing how DCR jointly strengthens and balances the D-Ability and P-Ability of the learned representations.

\subsection{A Naive Optimization Method}
\label{subsec: A Naive Optimization Method}

\noindent \textbf{How to improve \textcolor{DisColor}{D-Ability}?} Mainstream approaches commonly use a contrastive loss that pulls same-class features closer and pushes different-class features apart~\cite{radford2021learning}. Specifically, for a mini-batch $\mathcal{B}$ with cosine similarity $\mathrm{sim}(\z_i,\z_j)=\frac{\z_i^\top \z_j}{\|\z_i\|\,\|\z_j\|}$, temperature $\tau>0$, and the positive set $\mathcal{P}(i)=\{j\in\mathcal{B}\setminus\{i\}\mid y_j=y_i\}$, a InfoNCE-style loss~\cite{oord2018representation} is
\begin{equation}
\mathcal{L}_{\text{con}}=\frac{1}{|\mathcal{B}|}\sum_{i\in\mathcal{B}}-\log\frac{\sum_{j\in\mathcal{P}(i)} \exp\big(\mathrm{sim}(\z_i,\z_j)/\tau\big)}
{\sum_{k\in\mathcal{B}\setminus\{i\}} \exp\big(\mathrm{sim}(\z_i,\z_k)/\tau\big)},
\end{equation}
which decreases $S_{\text{inner}}$ and increases $S_{\text{inter}}$ in \cref{eq: minmax_s}.

\vskip 0.5 ex
\noindent \textbf{How to improve \textcolor{PerColor}{P-Ability}?} The objective is typically realized by enforcing consistency between the reconstructed image and the original one~\cite{wangdiffusion}, namely
\begin{equation}
\mathcal{L}_{\text{rec}}
=\mathbb{E}_{t}\Big[\big\|\epsilon_{\theta}\big(\x_t,\ h_{\omega}(f_{\phi}(\x)),\ t\big)-\boldsymbol\epsilon_t^{\text{gt}}\big\|_2^2\Big],
\end{equation}
which is exactly \cref{eq: min_rec}.

To obtain richer understanding representations without altering architectures or data pipelines, we consider a naive method that combines the two losses described above:
\begin{equation}
\mathcal{L}_{\text{joint}} =
\lambda_{\text{con}}\mathcal{L}_{\text{con}} +
\lambda_{\text{rec}}\mathcal{L}_{\text{rec}},
\label{eq: joint_loss}
\end{equation}
where $\lambda_{\text{con}}$ and $\lambda_{\text{rec}}$ are hyperparameters that balance the two objectives.

\vskip 0.5 ex
\noindent \textbf{Limitation}. Although this approach appears promising, we find in practice that it leads to a suboptimal solution, as shown in \cref{fig: intro overview}(d). As suggested in multi-task optimization~\cite{yu2020gradient,sener2018multi,lin2019pareto}, we suspect that \textbf{learning competition} may occur. The contrastive objective encourages class separation, while the reconstruction objective maintains image-level consistency. The mismatch between these two objectives causes the easier one to dominate the other. These conflicting forces often lead to oscillation or degeneration during training, resulting in unstable convergence or even feature collapse.

Specifically, let $\z$ denote the shared visual representation. At training step $t$, the update rule for $\z$ is:
\begin{equation}
\z^{t+1} = \z^{t} - \eta \nabla_{\z}\mathcal{L}_{\text{joint}}
= \z^{t} - \eta \left(
\lambda_{\text{con}}\nabla_{\z}\mathcal{L}_{\text{con}} +
\lambda_{\text{rec}}\nabla_{\z}\mathcal{L}_{\text{rec}}
\right),
\label{eq: joint_update}
\end{equation}
where $\eta$ is the learning rate. Denoting $\mathbf{g}_{\text{con}}=\nabla_{\z}\mathcal{L}_{\text{con}}$ and $\mathbf{g}_{\text{rec}}=\nabla_{\z}\mathcal{L}_{\text{rec}}$, the overall optimization direction depends on the geometric relation between these two gradients. If their directions are inconsistent, the optimization will exhibit \textit{competition} between the two objectives.

\vskip 0.5ex
\noindent\textbf{Gradient consistency and conflict.} To quantify this interaction, we define the cosine similarity between the two gradient directions:
\begin{equation}
\cos(\mathbf{g}_{\text{con}}, \mathbf{g}_{\text{rec}}) =
\frac{\mathbf{g}_{\text{con}}^\top \mathbf{g}_{\text{rec}}}
{\|\mathbf{g}_{\text{con}}\|_2 \|\mathbf{g}_{\text{rec}}\|_2}.
\label{eq: grad_cos}
\end{equation}
When $\cos(\mathbf{g}_{\text{con}}, \mathbf{g}_{\text{rec}}) > 0$, the two objectives cooperate and reinforce each other. However, when $\cos(\mathbf{g}_{\text{con}}, \mathbf{g}_{\text{rec}}) < 0$, they conflict, leading to destructive interference in the gradient updates. 

We conduct experiments on OpenAI CLIP ViT-L/14@224~\cite{radford2021learning} and optimize the model using the joint loss $\mathcal{L}_{\text{joint}}$. As shown in \cref{fig: Gradient Conflict}(a)(b), \textbf{$\mathcal{L}_{\text{con}}$ dominates the optimization process, while $\mathcal{L}_{\text{rec}}$ is suppressed and fails to converge.} \cref{fig: Gradient Conflict}(c) further shows that during training, $86.3\%$ of the steps have a negative cosine similarity $\cos(\mathbf{g}_{\text{con}}, \mathbf{g}_{\text{rec}})$ between the gradients of the two losses, indicating that \textbf{gradient conflict is pervasive throughout the process}. Moreover, as training progresses, this value becomes almost entirely negative, suggesting that gradient conflicts intensify in the later stages of optimization.

\begin{rem}
The above analysis reveals that naive weighted summation cannot effectively reconcile the conflicting learning objectives. Therefore, we argue that the problem should be addressed at its root by optimizing a single objective.
\end{rem}

\subsection{Diffusion Contrastive Reconstruction}
\label{subsec: Diffusion Contrastive Reconstruction}
We propose a new optimization framework, named \textbf{Diffusion Contrastive Reconstruction (DCR)}, that jointly optimizes the two optimization objectives (\cref{eq: minmax_s} and \cref{eq: min_rec}) within a unified formulation. With this single loss, the gradient conflict is naturally resolved.

\textbf{(1) Standard Diffusion-based Reconstruction.} As shown on the right side of \cref{fig: overview}, for an image $\x$, we sample a timestep $t$, obtain the forward noised latent $\tilde{\x}_t$, and condition the denoiser with the image-driven embedding $c=h_{\omega}(f_{\phi}(\tilde{\x}))$. The denoiser $\epsilon_{\theta}$ predicts the noise $\hat{\boldsymbol\epsilon}_{\theta}=\epsilon_{\theta}(\x_t,c,t)$ (abbreviated as $\hat{\boldsymbol\epsilon}$), which determines the next latent $\x_{t-1}$.

\textbf{(2) Injecting Contrastive Signals.} We construct a contrastive triplet in the denoising images, as illustrated on the left side of \cref{fig: overview}. Given an image $\x$, the \textbf{anchor} is the predicted noise when the image itself provides the condition, \textit{i.e.}, $\hat{\boldsymbol\epsilon}=\epsilon_{\theta}(\x_t,\,h_{\omega}(f_{\phi}(\tilde{\x})),\,t)$. We then apply an augmentation $\x^{+}=a(\x)$ (\textit{e.g.}, random crop or color jitter~\cite{chen2020simple,chen2020improved}), compute $c^{+}=h_{\omega}(f_{\phi}(\tilde{\x}^{+}))$, and define the \textbf{positive} as the noise predicted under this augmented condition, $\hat{\boldsymbol\epsilon}_{+}=\epsilon_{\theta}(\x_t,\,c^{+},\,t)$. All other images $\x^{j}$ ($j\neq i$) within the same mini-batch serve as \textbf{negatives}. For each $\x^{j}$ we set $c^{j}=h_{\omega}(f_{\phi}(\tilde{\x}^{\,j}))$ and obtain $\hat{\boldsymbol\epsilon}^{j}_{-}=\epsilon_{\theta}(\x_t,\,c^{j},\,t)$. Additionally, we include the ground-truth diffusion noise $\boldsymbol\epsilon_{t}^{\text{gt}}$ (abbreviated as $\boldsymbol\epsilon_{\text{gt}}$) as an auxiliary target to provide richer supervision.

\textbf{(3) DCR Loss.} Let $d(u,v)=\exp(\mathrm{sim}(u,v)/\tau)$, where $\mathrm{sim}(u,v)=\frac{\langle u,v\rangle}{|u|,|v|}$ denote cosine similarity and $\tau>0$ a temperature. Define set $P=\{\hat{\boldsymbol\epsilon}_{+},\boldsymbol\epsilon_{\text{gt}}\}$, set $N=\{\hat{\boldsymbol\epsilon}_{-}^{k}\}_{k=1}^{N-1}$, and set $C=P\cup N$. Then the loss is
\begin{equation}
\mathcal{L}_{\mathrm{dcr}}=-\frac{1}{2}\sum_{p\in P} \log
\frac{d(\hat{\boldsymbol\epsilon},p)}
{\sum_{c\in C}d(\hat{\boldsymbol\epsilon},c)}.
\label{eq: dcr_loss}
\end{equation}
Intuitively, the core of diffusion models lies in predicting noise. Performing contrastive learning on the predicting noise allows the model to sense fine-grained differences in the original visual representations used as conditions. This helps the vision encoder capture detailed information and enhance its discriminative ability, ultimately achieving a natural balance between \textbf{\textcolor{DisColor}{D-Ability}} and \textbf{\textcolor{PerColor}{P-Ability}}.

\vskip 0.5 ex
\noindent \textbf{Training Protocol.} During training, the generative model $\epsilon_{\theta}$ remains frozen. As shown in \cref{fig: overview}, we use a two-stage schedule to transfer diffusion knowledge into the vision encoder gradually.

\begin{table*}[t!]
  \centering
  \renewcommand\arraystretch{0.85}
  \caption{Performance of \textbf{\textcolor{PerColor}{Detail Perceptual Ability (P-Ability)}} on the MMVP-VLM benchmark. Results of baseline methods are taken from \cite{wangdiffusion,ma2025genhancer,li20252}. Our method outperforms across multiple CLIP backbones and exhibits robustness across various visual patterns. The visual patterns are symbolized as: \textbf{\faCompass}: Orientation and Direction, \textbf{\faSearch}: Presence of Specific Features, \textbf{\faSync}: State and Condition, \textbf{\faSortNumericUp}: Quantity and Count, \textbf{\faMapPin}: Positional and Relational Context, \textbf{\faPalette}: Color and Appearance, \textbf{\faCogs}: Structural and Physical Characteristics, \textbf{\faFont}: Texts, \textbf{\faCamera}: Viewpoint and Perspective.}
  \resizebox{0.9\linewidth}{!}{
    \begin{tabular}{cccc|ccccccccc|c}
    \toprule
    \textbf{CLIP Backbone} & \textbf{Resol.} & \textbf{\#Params} & \textbf{Method} & \textbf{\faCompass} & \textbf{\faSearch} & \textbf{\faSync} & \textbf{\faSortNumericUp} & \textbf{\faMapPin} & \textbf{\faPalette} & \textbf{\faCogs} & \textbf{\faFont} & \textbf{\faCamera} & \textbf{Avg} \\
    \midrule
    \multirow{5}[1]{*}{OpenAI ViT-L-14} & \multirow{5}[1]{*}{224} & \multirow{5}[1]{*}{427.6M} & Original & 13.3  & 13.3  & 20.0  & 20.0  & 13.3  & 53.3  & 20.0  & 6.7   & 13.3  & 19.2  \\
          &       &       & DIVA  & 13.3  & 20.0  & 40.0  & 6.7   & 20.0  & 53.3  & 46.7  & 20.0  & 13.3  & 25.9  \\
          &       &       & GenHancer & 13.3  & 33.3  & 33.3  & 20.0  & 6.7   & 73.3  & 46.7  & 20.0  & 40.0  & 31.8  \\
          &       &       & un$^2$CLIP & 0.0   & 33.3  & 46.7  & 26.7  & 13.3  & 80.0  & 40.0  & 20.0  & 33.3  & 32.6  \\
          &       &       & \cellcolor[rgb]{ .851,  .882,  .957}Ours & \cellcolor[rgb]{ .851,  .882,  .957}13.3  & \cellcolor[rgb]{ .851,  .882,  .957}33.3  & \cellcolor[rgb]{ .851,  .882,  .957}33.3  & \cellcolor[rgb]{ .851,  .882,  .957}26.7  & \cellcolor[rgb]{ .851,  .882,  .957}6.7  & \cellcolor[rgb]{ .851,  .882,  .957}73.3  & \cellcolor[rgb]{ .851,  .882,  .957}46.7  & \cellcolor[rgb]{ .851,  .882,  .957}20.0  & \cellcolor[rgb]{ .851,  .882,  .957}46.7  & \cellcolor[rgb]{ .851,  .882,  .957}\textbf{33.3}  \\
    \midrule
    \multirow{5}[1]{*}{OpenAI ViT-L-14} & \multirow{5}[1]{*}{336} & \multirow{5}[1]{*}{427.9M} & Original & 0.0   & 20.0  & 40.0  & 20.0  & 6.7   & 20.0  & 33.3  & 6.7   & 33.3  & 20.0  \\
          &       &       & DIVA  & 26.7  & 20.0  & 33.3  & 13.3  & 13.3  & 46.7  & 26.7  & 6.7   & 40.0  & 25.2  \\
          &       &       & GenHancer & 6.7   & 20.0  & 33.3  & 20.0  & 6.7   & 73.3  & 53.3  & 26.7  & 26.7  & 29.6  \\
          &       &       & un$^2$CLIP & 6.7   & 33.3  & 46.7  & 13.3  & 13.3  & 80.0  & 40.0  & 20.0  & 20.0  & 30.4  \\
          &       &       & \cellcolor[rgb]{ .851,  .882,  .957}Ours & \cellcolor[rgb]{ .851,  .882,  .957}13.3  & \cellcolor[rgb]{ .851,  .882,  .957}26.7  & \cellcolor[rgb]{ .851,  .882,  .957}33.3  & \cellcolor[rgb]{ .851,  .882,  .957}20.0  & \cellcolor[rgb]{ .851,  .882,  .957}6.7  & \cellcolor[rgb]{ .851,  .882,  .957}73.3  & \cellcolor[rgb]{ .851,  .882,  .957}46.7  & \cellcolor[rgb]{ .851,  .882,  .957}33.3  & \cellcolor[rgb]{ .851,  .882,  .957}26.7  & \cellcolor[rgb]{ .851,  .882,  .957}\textbf{31.1}  \\
    \midrule
    \midrule
    \multirow{5}[1]{*}{MetaCLIP ViT-L-14} & \multirow{5}[1]{*}{224} & \multirow{5}[1]{*}{427.6M} & Original & 13.3  & 6.7   & 66.7  & 6.7   & 33.3  & 46.7  & 20.0  & 6.7   & 13.3  & 23.7  \\
          &       &       & DIVA  & 6.7   & 6.7   & 60.0  & 0.0   & 26.7  & 66.7  & 20.0  & 20.0  & 40.0  & 27.4  \\
          &       &       & GenHancer & 13.3  & 20.0  & 53.3  & 13.3  & 26.7  & 80.0  & 33.3  & 13.3  & 33.3  & 31.8  \\
          &       &       & un$^2$CLIP & -     & -     & -     & -     & -     & -     & -     & -     & -     & - \\
          &       &       & \cellcolor[rgb]{ .851,  .882,  .957}Ours & \cellcolor[rgb]{ .851,  .882,  .957}20.0 & \cellcolor[rgb]{ .851,  .882,  .957}20.0 & \cellcolor[rgb]{ .851,  .882,  .957}53.3 & \cellcolor[rgb]{ .851,  .882,  .957}13.3 & \cellcolor[rgb]{ .851,  .882,  .957}26.7 & \cellcolor[rgb]{ .851,  .882,  .957}80.0 & \cellcolor[rgb]{ .851,  .882,  .957}33.3 & \cellcolor[rgb]{ .851,  .882,  .957}13.3 & \cellcolor[rgb]{ .851,  .882,  .957}33.3 & \cellcolor[rgb]{ .851,  .882,  .957}\textbf{32.6} \\
    \midrule
    \multirow{5}[1]{*}{MetaCLIP ViT-H-14} & \multirow{5}[1]{*}{224} & \multirow{5}[1]{*}{986.1M} & Original & 6.7   & 13.3  & 60.0  & 13.3  & 6.7   & 53.3  & 26.7  & 13.3  & 33.3  & 25.2  \\
          &       &       & DIVA  & 13.3  & 20.0  & 53.3  & 33.3  & 13.3  & 66.7  & 33.3  & 13.3  & 40.0  & 31.8  \\
          &       &       & GenHancer & 20.0  & 20.0  & 66.7  & 26.7  & 26.7  & 66.7  & 33.3  & 20.0  & 53.3  & 37.0  \\
          &       &       & un$^2$CLIP & -     & -     & -     & -     & -     & -     & -     & -     & -     & - \\
          &       &       & \cellcolor[rgb]{ .851,  .882,  .957}Ours & \cellcolor[rgb]{ .851,  .882,  .957}20.0 & \cellcolor[rgb]{ .851,  .882,  .957}20.0 & \cellcolor[rgb]{ .851,  .882,  .957}66.7 & \cellcolor[rgb]{ .851,  .882,  .957}26.7 & \cellcolor[rgb]{ .851,  .882,  .957}26.7 & \cellcolor[rgb]{ .851,  .882,  .957}66.7 & \cellcolor[rgb]{ .851,  .882,  .957}33.3 & \cellcolor[rgb]{ .851,  .882,  .957}26.7 & \cellcolor[rgb]{ .851,  .882,  .957}53.3 & \cellcolor[rgb]{ .851,  .882,  .957}\textbf{37.8} \\
    \midrule
    \midrule
    \multirow{5}[1]{*}{SigLIP ViT-SO-14} & \multirow{5}[1]{*}{224} & \multirow{5}[1]{*}{877.4M} & Original & 26.7  & 20.0  & 53.3  & 40.0  & 20.0  & 66.7  & 40.0  & 20.0  & 53.3  & 37.8  \\
          &       &       & DIVA  & 13.3  & 26.7  & 60.0  & 46.7  & 13.3  & 73.3  & 53.3  & 26.7  & 53.3  & 40.7  \\
          &       &       & GenHancer & 20.0  & 20.0  & 66.7  & 60.0  & 20.0  & 86.7  & 40.0  & 13.0  & 53.3  & 42.2  \\
          &       &       & un$^2$CLIP & -     & -     & -     & -     & -     & -     & -     & -     & -     & - \\
          &       &       & \cellcolor[rgb]{ .851,  .882,  .957}Ours & \cellcolor[rgb]{ .851,  .882,  .957}13.3 & \cellcolor[rgb]{ .851,  .882,  .957}26.7 & \cellcolor[rgb]{ .851,  .882,  .957}73.3 & \cellcolor[rgb]{ .851,  .882,  .957}60.0 & \cellcolor[rgb]{ .851,  .882,  .957}20.0 & \cellcolor[rgb]{ .851,  .882,  .957}80.0 & \cellcolor[rgb]{ .851,  .882,  .957}40.0 & \cellcolor[rgb]{ .851,  .882,  .957}20.0 & \cellcolor[rgb]{ .851,  .882,  .957}53.3 & \cellcolor[rgb]{ .851,  .882,  .957}\textbf{43.0} \\
    \midrule
    \multirow{5}[1]{*}{SigLIP ViT-SO-14} & \multirow{5}[1]{*}{384} & \multirow{5}[1]{*}{878.0M} & Original & 20.0  & 26.7  & 60.0  & 33.3  & 13.3  & 66.7  & 33.3  & 26.7  & 53.3  & 37.0  \\
          &       &       & DIVA  & 26.7  & 33.3  & 53.3  & 26.7  & 13.3  & 80.0  & 40.0  & 26.7  & 46.7  & 38.5  \\
          &       &       & GenHancer & 26.7  & 20.0  & 66.7  & 33.3  & 13.3  & 86.7  & 40.0  & 26.7  & 46.7  & 40.0  \\
          &       &       & un$^2$CLIP & 20.0  & 20.0  & 60.0  & 46.7  & 26.7  & 73.3  & 40.0  & 26.7  & 60.0  & 41.5  \\
          &       &       & \cellcolor[rgb]{ .851,  .882,  .957}Ours & \cellcolor[rgb]{ .851,  .882,  .957}26.7 & \cellcolor[rgb]{ .851,  .882,  .957}33.3 & \cellcolor[rgb]{ .851,  .882,  .957}60.0 & \cellcolor[rgb]{ .851,  .882,  .957}33.3 & \cellcolor[rgb]{ .851,  .882,  .957}26.7 & \cellcolor[rgb]{ .851,  .882,  .957}80.0 & \cellcolor[rgb]{ .851,  .882,  .957}40.0 & \cellcolor[rgb]{ .851,  .882,  .957}33.3 & \cellcolor[rgb]{ .851,  .882,  .957}46.7 & \cellcolor[rgb]{ .851,  .882,  .957}\textbf{42.2} \\
    \bottomrule
    \end{tabular}%
    }
\vspace{-10pt}
\label{tab: understanding ability results}
\end{table*}

\begin{itemize}
\item \textbf{Stage-1 (projector alignment).} We freeze the visual encoder $f_{\phi}$ and train only the projector $h_{\omega}$. This stage learns a conditional mapping that aligns the visual guidance with the text guidance in the original diffusion model, ensuring that the frozen denoiser can correctly interpret the image-based conditions.

\item \textbf{Stage-2 (encoder enhancement).} We then freeze the projector $h_{\omega}$ and train the visual encoder $f_{\phi}$. Since $h_{\omega}$ has been aligned to the denoiser in Stage-1, gradients from \cref{eq: dcr_loss} now directly refine the feature structure produced by $f_{\phi}$. Features from the same class are guided to produce conditions that yield similar predicted noises, while features from different classes are guided to produce conditions with dissimilar predicted noises. After this stage, we obtain a CLIP visual encoder with richer visual representations.
\end{itemize}

\subsection{Theoretical Statements}
\label{subsec: Theoretical Statements}
Taking a step further, we conduct a theoretical analysis. The results show that the DCR loss is designed to satisfy the two objectives discussed in \cref{subsec: Visual Representation Learning}. We present these insights in the following two theorems. Due to space limitations, we provide a concise version here, and the proofs are deferred to the Appendix.

\begin{thm}
Fix a diffusion step $t$ and define $T(\z)=\epsilon_{\theta}\bigl(x_t,\,h_{\omega}(\z),\,t\bigr)$. Under a mild assumption, the intra-class scatter and inter-class scatter in the \emph{feature} space ($S_{\text{inner}}, S_{\text{inter}}$) can be bounded by those in the \emph{noise} space ($S_{\text{inner}}^{(\epsilon)}(t), S_{\text{inter}}^{(\epsilon)}(t)$) as
\begin{equation}
S_{\text{inner}}\le\frac{1}{m^2}\,S_{\text{inner}}^{(\epsilon)}(t),
\end{equation}
\begin{equation}
S_{\text{inter}}\ge\kappa S_{\text{inter}}^{(\epsilon)}(t)-\eta S_{\text{inner}}^{(\epsilon)}(t),
\end{equation}
where $m$, $\kappa$, and $\eta$ are positive constants depending on the Lipschitz continuity of $T(\cdot)$.
\label{thm: D-Ability short theorem}
\end{thm}

\begin{thm}
If negatives in the mini-batch are well separated from the anchor (\textit{i.e.}, $\mathrm{sim}(\hat{\boldsymbol\epsilon},\hat{\boldsymbol\epsilon}^{j}_{-})\ll \mathrm{sim}(\hat{\boldsymbol\epsilon},\boldsymbol\epsilon_{\text{gt}})$ for all $j\neq i$), and predicted-noise norms are bounded away from $0$ and $\infty$, then the DCR loss reduces to a scaled reconstruction loss up to an additive constant:
\begin{equation}
\mathcal{L}_{\mathrm{dcr}}=\lambda\big\|\epsilon_{\theta}(\x_t,\ h_{\omega}(f_{\phi}(\tilde{\x})),t)-\boldsymbol\epsilon_t^{\text{gt}}\big\|_2^2 + c,
\end{equation}
where $\lambda>0$, and $c$ is a constant.
\label{thm: P-Ability short theorem}
\end{thm}

\begin{rem}
According to Theorem \ref{thm: D-Ability short theorem}, as $\mathcal{L}_{\mathrm{dcr}}$ decreases, we obtain $S_{\text{inner}}\downarrow$ and $S_{\text{inter}}\uparrow$ in the CLIP feature space, which fulfills the discriminative objective in \cref{eq: minmax_s}. According to Theorem \ref{thm: P-Ability short theorem}, minimizing $\mathcal{L}_{\mathrm{dcr}}$ becomes equivalent to minimizing the detail perceptual objective in \cref{eq: min_rec}.
\end{rem}

\begin{table*}[t!]
  \centering
  \renewcommand\arraystretch{1}
  \caption{Performance of \textbf{\textcolor{DisColor}{Discriminative Ability (D-Ability)}} on 6 standard zero-shot clustering benchmarks. Our method achieves better class separability. O-1, M-1, and S-1 separately represent OpenAI CLIP ViT-L@224, MetaCLIP ViT-L@224, and SigLIP ViT-SO@224.}
  \resizebox{\linewidth}{!}{
    \begin{tabular}{c|c|cccccccccccccccccc|ccc}
    \toprule
    \multirow{2}[4]{*}{\textbf{Backbone}} & \multirow{2}[4]{*}{\textbf{Method}} & \multicolumn{3}{c}{\textbf{MNIST}~\cite{lecun2002gradient}} & \multicolumn{3}{c}{\textbf{CIFAR-10}~\cite{krizhevsky2009learning}} & \multicolumn{3}{c}{\textbf{Eurosat}~\cite{helber2019eurosat}} & \multicolumn{3}{c}{\textbf{Caltech-101}~\cite{fei2004learning}} & \multicolumn{3}{c}{\textbf{DTD}~\cite{cimpoi2014describing}} & \multicolumn{3}{c|}{\textbf{ImageNet-1K}~\cite{deng2009imagenet}} & \multicolumn{3}{c}{\textbf{Avg}} \\
    \cmidrule(lr){3-5}\cmidrule(lr){6-8}\cmidrule(lr){9-11}\cmidrule(lr){12-14}\cmidrule(lr){15-17}\cmidrule(lr){18-20}\cmidrule(lr){21-23}
    & & NMI & ACC & ARI & NMI & ACC & ARI & NMI & ACC & ARI & NMI & ACC & ARI & NMI & ACC & ARI & NMI & ACC & ARI & NMI & ACC & ARI \\
    \midrule
    \multirow{5}[2]{*}{O-1} & Original & 0.63  & 0.63  & 0.50  & 0.82  & 0.82  & 0.76  & 0.71  & 0.72  & 0.61  & 0.78  & 0.53  & 0.35  & 0.55  & 0.45  & 0.30  & 0.79  & 0.53  & 0.41  & 0.71  & 0.61  & 0.49  \\
          & DIVA  & 0.66  & 0.60  & 0.49  & 0.81  & 0.82  & 0.75  & 0.70  & 0.72  & 0.61  & 0.78  & 0.52  & 0.34  & 0.51  & 0.41  & 0.26  & 0.79  & 0.52  & 0.40  & 0.71  & 0.60  & 0.48  \\
          & GenHancer & 0.69  & 0.73  & 0.60  & 0.58  & 0.65  & 0.49  & 0.61  & 0.67  & 0.56  & 0.81  & 0.54  & 0.38  & 0.56  & 0.45  & 0.31  & 0.64  & 0.26  & 0.14  & 0.65  & 0.55  & 0.41  \\
          & un$^2$CLIP & 0.67  & 0.70  & 0.53  & 0.81  & 0.82  & 0.74  & 0.64  & 0.75  & 0.60  & 0.80  & 0.55  & 0.41  & 0.55  & 0.46  & 0.31  & 0.73  & 0.41  & 0.28  & 0.70  & 0.62  & 0.48  \\
          & \cellcolor[rgb]{ .851,  .882,  .957}Ours & \cellcolor[rgb]{ .851,  .882,  .957}0.73  & \cellcolor[rgb]{ .851,  .882,  .957}0.74  & \cellcolor[rgb]{ .851,  .882,  .957}0.62  & \cellcolor[rgb]{ .851,  .882,  .957}0.89  & \cellcolor[rgb]{ .851,  .882,  .957}0.86  & \cellcolor[rgb]{ .851,  .882,  .957}0.80  & \cellcolor[rgb]{ .851,  .882,  .957}0.68  & \cellcolor[rgb]{ .851,  .882,  .957}0.73  & \cellcolor[rgb]{ .851,  .882,  .957}0.61  & \cellcolor[rgb]{ .851,  .882,  .957}0.83  & \cellcolor[rgb]{ .851,  .882,  .957}0.57  & \cellcolor[rgb]{ .851,  .882,  .957}0.43  & \cellcolor[rgb]{ .851,  .882,  .957}0.58  & \cellcolor[rgb]{ .851,  .882,  .957}0.48  & \cellcolor[rgb]{ .851,  .882,  .957}0.33  & \cellcolor[rgb]{ .851,  .882,  .957}0.84  & \cellcolor[rgb]{ .851,  .882,  .957}0.66  & \cellcolor[rgb]{ .851,  .882,  .957}0.45  & \cellcolor[rgb]{ .851,  .882,  .957}\textbf{0.76} & \cellcolor[rgb]{ .851,  .882,  .957}\textbf{0.67} & \cellcolor[rgb]{ .851,  .882,  .957}\textbf{0.54} \\
    \midrule
    \multirow{4}[2]{*}{M-1} & Original & 0.59  & 0.62  & 0.47  & 0.79  & 0.73  & 0.67  & 0.59  & 0.56  & 0.46  & 0.79  & 0.52  & 0.40  & 0.56  & 0.51  & 0.36  & 0.81  & 0.56  & 0.44  & 0.69  & 0.58  & 0.47  \\
          & DIVA  & 0.60  & 0.65  & 0.49  & 0.77  & 0.79  & 0.72  & 0.65  & 0.75  & 0.56  & 0.80  & 0.56  & 0.39  & 0.60  & 0.52  & 0.36  & 0.81  & 0.56  & 0.44  & 0.71  & 0.64  & 0.49  \\
          & GenHancer & 0.76  & 0.73  & 0.77  & 0.70  & 0.70  & 0.60  & 0.62  & 0.72  & 0.65  & 0.81  & 0.55  & 0.41  & 0.59  & 0.49  & 0.34  & 0.64  & 0.27  & 0.15  & 0.69  & 0.58  & 0.49  \\
          & \cellcolor[rgb]{ .851,  .882,  .957}Ours & \cellcolor[rgb]{ .851,  .882,  .957}0.78  & \cellcolor[rgb]{ .851,  .882,  .957}0.79  & \cellcolor[rgb]{ .851,  .882,  .957}0.68  & \cellcolor[rgb]{ .851,  .882,  .957}0.80  & \cellcolor[rgb]{ .851,  .882,  .957}0.82  & \cellcolor[rgb]{ .851,  .882,  .957}0.76  & \cellcolor[rgb]{ .851,  .882,  .957}0.72  & \cellcolor[rgb]{ .851,  .882,  .957}0.77  & \cellcolor[rgb]{ .851,  .882,  .957}0.67  & \cellcolor[rgb]{ .851,  .882,  .957}0.82  & \cellcolor[rgb]{ .851,  .882,  .957}0.57  & \cellcolor[rgb]{ .851,  .882,  .957}0.42  & \cellcolor[rgb]{ .851,  .882,  .957}0.60  & \cellcolor[rgb]{ .851,  .882,  .957}0.54  & \cellcolor[rgb]{ .851,  .882,  .957}0.37  & \cellcolor[rgb]{ .851,  .882,  .957}0.84  & \cellcolor[rgb]{ .851,  .882,  .957}0.69  & \cellcolor[rgb]{ .851,  .882,  .957}0.45  & \cellcolor[rgb]{ .851,  .882,  .957}\textbf{0.76} & \cellcolor[rgb]{ .851,  .882,  .957}\textbf{0.70} & \cellcolor[rgb]{ .851,  .882,  .957}\textbf{0.56} \\
    \midrule
    \multirow{4}[2]{*}{S-1} & Original & 0.68  & 0.74  & 0.56  & 0.94  & 0.97  & 0.95  & 0.64  & 0.69  & 0.55  & 0.89  & 0.61  & 0.43  & 0.66  & 0.55  & 0.40  & 0.85  & 0.63  & 0.52  & 0.78  & 0.70  & 0.57  \\
          & DIVA  & 0.63  & 0.61  & 0.46  & 0.94  & 0.98  & 0.95  & 0.66  & 0.73  & 0.57  & 0.90  & 0.66  & 0.48  & 0.66  & 0.55  & 0.41  & 0.85  & 0.63  & 0.52  & 0.77  & 0.69  & 0.57  \\
          & GenHancer & 0.80  & 0.78  & 0.70  & 0.72  & 0.69  & 0.61  & 0.65  & 0.70  & 0.57  & 0.88  & 0.64  & 0.48  & 0.62  & 0.55  & 0.42  & 0.77  & 0.48  & 0.35  & 0.74  & 0.64  & 0.52  \\
          & \cellcolor[rgb]{ .851,  .882,  .957}Ours & \cellcolor[rgb]{ .851,  .882,  .957}0.85  & \cellcolor[rgb]{ .851,  .882,  .957}0.89  & \cellcolor[rgb]{ .851,  .882,  .957}0.78  & \cellcolor[rgb]{ .851,  .882,  .957}0.95  & \cellcolor[rgb]{ .851,  .882,  .957}0.98  & \cellcolor[rgb]{ .851,  .882,  .957}0.95  & \cellcolor[rgb]{ .851,  .882,  .957}0.76  & \cellcolor[rgb]{ .851,  .882,  .957}0.81  & \cellcolor[rgb]{ .851,  .882,  .957}0.69  & \cellcolor[rgb]{ .851,  .882,  .957}0.89  & \cellcolor[rgb]{ .851,  .882,  .957}0.66  & \cellcolor[rgb]{ .851,  .882,  .957}0.48  & \cellcolor[rgb]{ .851,  .882,  .957}0.67  & \cellcolor[rgb]{ .851,  .882,  .957}0.59  & \cellcolor[rgb]{ .851,  .882,  .957}0.44  & \cellcolor[rgb]{ .851,  .882,  .957}0.88  & \cellcolor[rgb]{ .851,  .882,  .957}0.65  & \cellcolor[rgb]{ .851,  .882,  .957}0.55  & \cellcolor[rgb]{ .851,  .882,  .957}\textbf{0.83} & \cellcolor[rgb]{ .851,  .882,  .957}\textbf{0.76} & \cellcolor[rgb]{ .851,  .882,  .957}\textbf{0.65} \\
    \bottomrule
    \end{tabular}%
   }
   \vspace{-15pt}
\label{tab: discriminative ability results}
\end{table*}
\section{Experiments}
\label{sec: experiments}

In this section, we describe some details of the experiments and present our results. \textbf{Due to space limitations, please refer to the Appendix for an extended version.}

\subsection{Experimental Setups}
\label{subsec: Experimental Setups}

\begin{figure}[t]
  \centering
   \includegraphics[width=\linewidth]{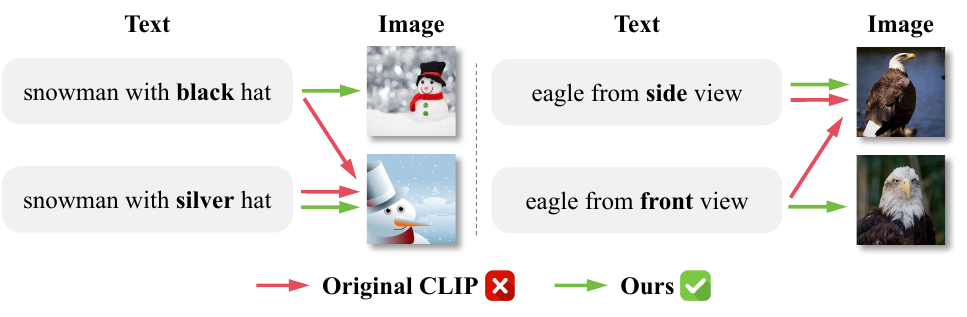}
   \vspace{-20pt}
   \caption{Qualitative results of \textbf{\textcolor{PerColor}{P-Ability}} on the MMVP-VLM benchmark. The predictions from the original CLIP and our improved version are indicated by red and green arrows, respectively. The improved CLIP effectively addresses the original model's limitations in capturing fine-grained visual details.}
   \vspace{-15pt}
   \label{fig: understanding ability qualitative results}
\end{figure}

\noindent \textbf{Implementation Details.} We perform all experiments on an NVIDIA-A100 80GB GPU. In the main experiments, Stable Diffusion v2.1~\cite{rombach2022high} serves as the diffusion backbone, and a simple two-layer MLP is used as the projector. Following \cite{ma2025genhancer}, we discard all patch tokens and retain only the \texttt{[CLS]} token as the conditional input. The training is conducted on the CC3M~\cite{sharma2018conceptual} dataset. We adopt the \textit{Adam with Weight Decay (AdamW)}~\cite{loshchilov2018fixing} optimizer with a weight decay of $0.01$. The initial learning rates are set to $1\times10^{-4}$ and $1\times10^{-5}$ for Stage-1 and Stage-2, respectively. In Stage-2, we fine-tune LoRA~\cite{hulora} with a rank of $16$ applied to the vision encoder. The batch size is fixed at $16$ for all experiments, and the total number of training steps is $4600$.

\vskip 0.5 ex
\noindent \textbf{Competitors.} Our method enhances visual representation through diffusion-based reconstruction. We compare it with three state-of-the-art methods: DIVA~\cite{wangdiffusion}, GenHancer~\cite{ma2025genhancer}, and un$^2$CLIP~\cite{li20252}. Among them, GenHancer and un$^2$CLIP retrain their own generative models, whereas our method is built on top of the pretrained Stable Diffusion.

\vskip 0.5 ex
\noindent \textbf{Evaluation Protocol.} Following \cite{wangdiffusion,ma2025genhancer}, we enhance $6$ CLIP backbones: OpenAI CLIP ViT-L@224/@336~\cite{radford2021learning}, MetaCLIP ViT-L/H@224~\cite{xudemystifying}, and SigLIP ViT-SO@224/@384~\cite{zhai2023sigmoid}. To evaluate the \textbf{\textcolor{PerColor}{P-Ability}}, we use MMVP-VLM~\cite{tong2024eyes}. To evaluate the \textbf{\textcolor{DisColor}{D-Ability}}, we perform zero-shot clustering on $6$ datasets~\cite{lecun2002gradient,krizhevsky2009learning,helber2019eurosat,fei2004learning,cimpoi2014describing,deng2009imagenet} and adopt \textit{Normalized Mutual Information (NMI)}, \textit{Accuracy (ACC)}, and \textit{Adjusted Rand Index (ARI)} to assess performance. We also train MLLMs with our enhanced CLIP encoders using the official LLaVA-1.5~\cite{liu2024improved} recipe. We then evaluate these MLLMs on the same benchmarks~\cite{tong2024eyes,tong2024cambrian,li2024naturalbench,li2023evaluating,lu2022learn,guan2024hallusionbench} as \cite{ma2025genhancer}.

\subsection{Overall Performance}
\label{subsec: Overall Performance}

\begin{figure}
  \centering
  \begin{subfigure}{0.32\linewidth}
    \includegraphics[width=\linewidth]{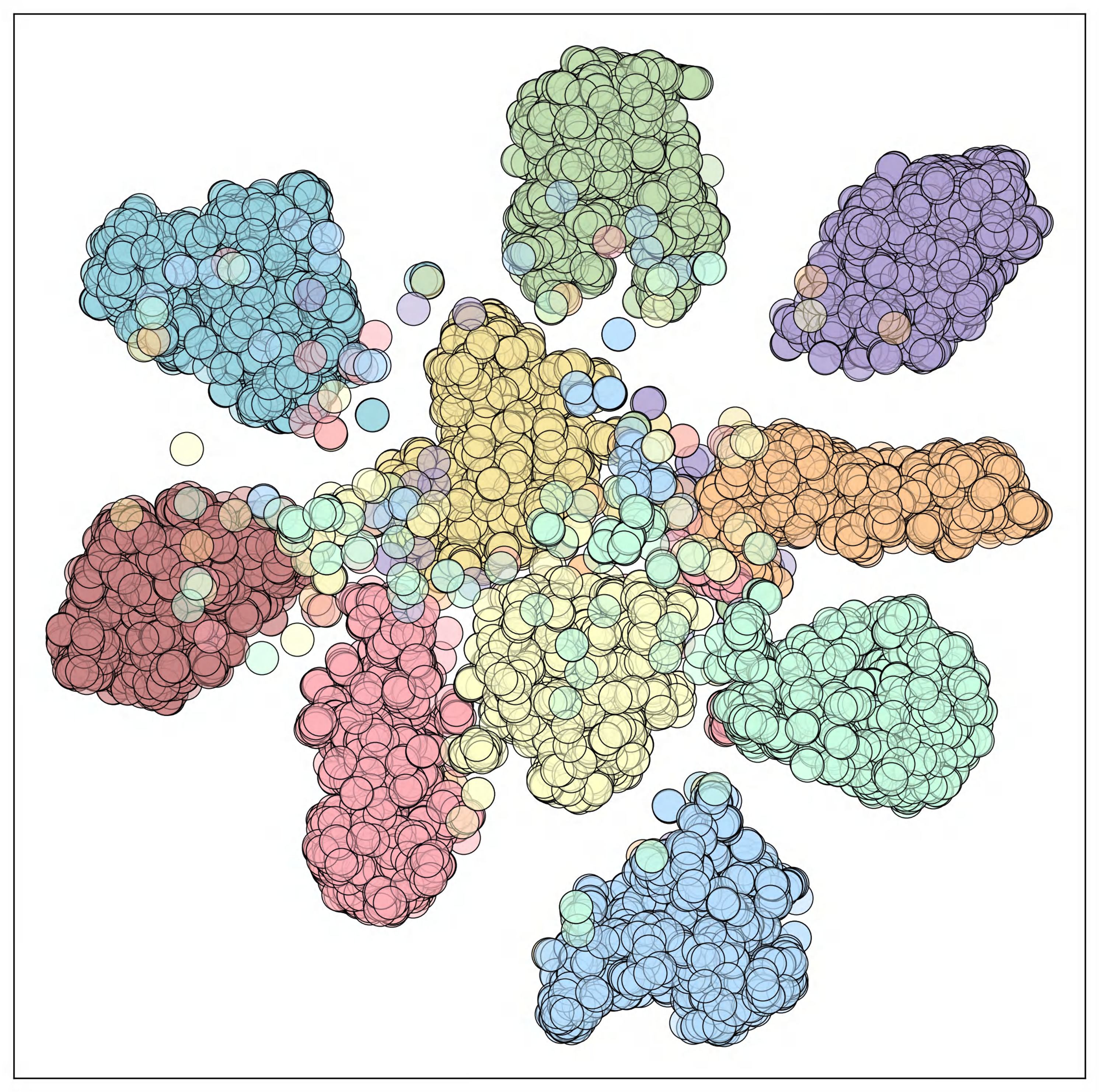}
    \caption{Original CLIP}
  \end{subfigure}
  \hfill
  \begin{subfigure}{0.32\linewidth}
    \includegraphics[width=\linewidth]{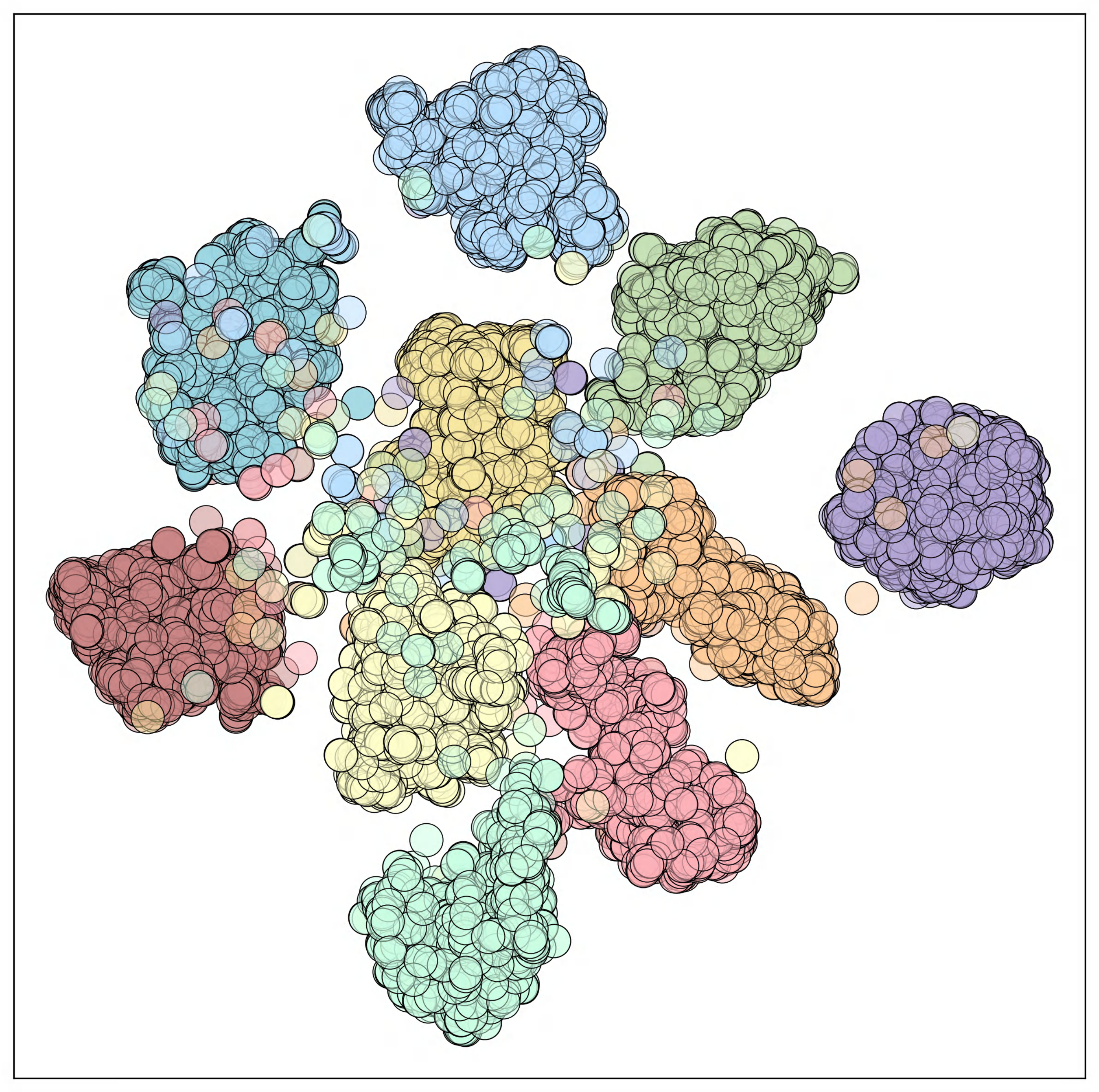}
    \caption{DIVA}
  \end{subfigure}
  \hfill
  \begin{subfigure}{0.32\linewidth}
    \includegraphics[width=\linewidth]{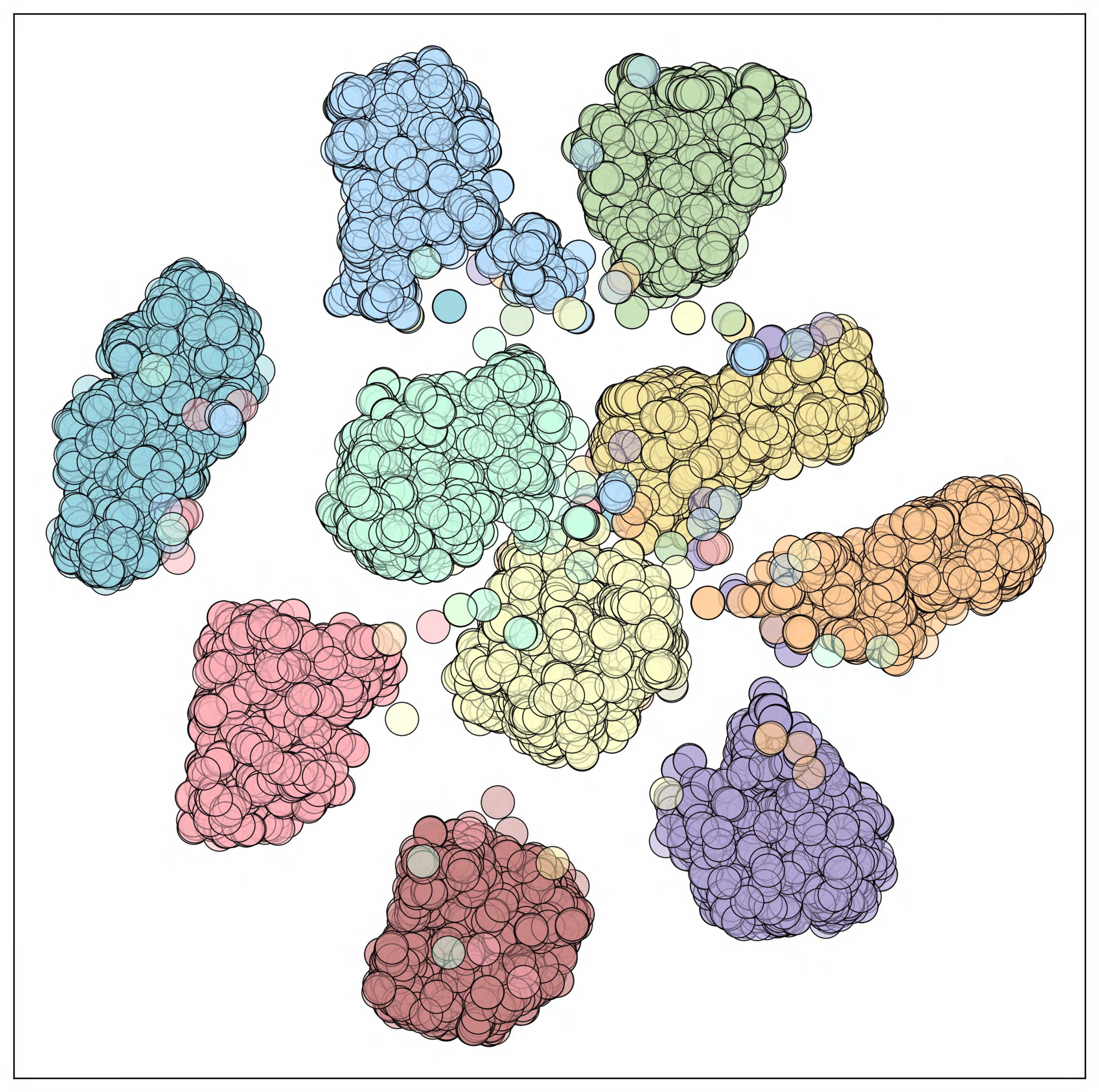}
    \caption{Ours}
  \end{subfigure}
  \vspace{-5pt}
  \caption{Qualitative results of \textbf{\textcolor{DisColor}{D-Ability}} on the MNIST benchmark by using the t-SNE method. The improved CLIP achieves better class separability.}
  \vspace{-15pt}
  \label{fig: discriminative ability qualitative results}
\end{figure}

\begin{table*}[t!]
    \centering
    \renewcommand\arraystretch{0.9}
    \caption{Performance of MLLMs (LLaVA-1.5~\cite{liu2024improved}) on various benchmarks. The champion and the runner-up are highlighted in \textbf{bold} and \underline{underline}. Results on NaturalBench follow the setting in \cite{ma2025genhancer}, which differs from that used in un$^2$CLIP~\cite{li20252}, leading to the missing entries.}
    \resizebox{\linewidth}{!}{
    \begin{tabular}{@{}llccccccccccccc@{}}
    \toprule
    \multirow{3}{*}{LLM} & \multirow{3}{*}{CLIP} & \multicolumn{8}{c}{Vision-Centric Benchmarks} & \multicolumn{5}{c}{Conventional MLLM Benchmarks} \\ \cmidrule(l){3-10} \cmidrule(l){11-15} 
     &  & \multirow{2}{*}{\begin{tabular}[c]{@{}c@{}}MMVP-\\ MLLM~\cite{tong2024eyes}\end{tabular}} & \multicolumn{4}{c}{NaturalBench~\cite{li2024naturalbench}} & \multicolumn{2}{c}{CV-Bench 2D~\cite{tong2024cambrian}} & \multirow{2}{*}{\begin{tabular}[c]{@{}c@{}}CV-Bench\\ 3D~\cite{tong2024cambrian}\end{tabular}} & \multicolumn{3}{c}{POPE~\cite{li2023evaluating}} & \multirow{2}{*}{\begin{tabular}[c]{@{}c@{}}SciQA-\\ IMG~\cite{lu2022learn}\end{tabular}} & \multirow{2}{*}{\begin{tabular}[c]{@{}c@{}}Hallusion\\ Avg.~\cite{guan2024hallusionbench}\end{tabular}} \\ \cmidrule(lr){4-7} \cmidrule(lr){8-9} \cmidrule(lr){11-13}
     &  &  & Acc & Q-Acc & I-Acc & G-Acc & ADE20K & COCO &  & rand & pop & adv &  &  \\ \midrule
    \multirow{5}{*}{Vicuna-7B~~~} & Original & 24.7 & 76.4 & 53.6 & 56.4 & 17.6 & 49.6 & 60.9 & 58.7 & 87.3 & 86.1 & 84.2 & 66.8 & 27.6 \\
     & DIVA & \textbf{31.3} & 75.3 & 51.7 & 56.1 & 22.3 & 51.3 & 63.4 & 60.2 & 87.9 & \underline{87.0} & 84.6 & 66.3 & \textbf{28.6} \\
     & GenHancer & \underline{30.7} & \underline{77.3} & \underline{55.6} & \underline{59.1} & \underline{24.4} & 52.9 & 63.6 & \underline{63.2} & \textbf{88.1} & 86.7 & 84.6 & 66.5 & \underline{28.4} \\ 
     & un$^2$CLIP & \textbf{31.3} & - & - & - & - & \underline{53.9} & \underline{65.1} & 61.2 & \underline{88.0} & \textbf{87.4} & \underline{85.4} & \textbf{68.4} & \underline{28.4} \\ 
     & \cellcolor[rgb]{ .851,  .882,  .957}Ours & \cellcolor[rgb]{ .851,  .882,  .957}\textbf{31.3} & \cellcolor[rgb]{ .851,  .882,  .957}\textbf{78.6} & \cellcolor[rgb]{ .851,  .882,  .957}\textbf{57.8} & \cellcolor[rgb]{ .851,  .882,  .957}\textbf{61.2} & \cellcolor[rgb]{ .851,  .882,  .957}\textbf{25.3} & \cellcolor[rgb]{ .851,  .882,  .957}\textbf{54.1} & \cellcolor[rgb]{ .851,  .882,  .957}\textbf{65.5} & \cellcolor[rgb]{ .851,  .882,  .957}\textbf{63.5} & \cellcolor[rgb]{ .851,  .882,  .957}\underline{88.0} & \cellcolor[rgb]{ .851,  .882,  .957}\textbf{87.4} & \cellcolor[rgb]{ .851,  .882,  .957}\textbf{85.8} & \cellcolor[rgb]{ .851,  .882,  .957}\underline{67.5} & \cellcolor[rgb]{ .851,  .882,  .957}\textbf{28.6} \\ 
     \bottomrule
    \end{tabular}
    }
    \vspace{-10pt}
    \label{tab: mllm results}
\end{table*}

\noindent \textbf{Performance of P-Ability.} Our method is evaluated on the MMVP-VLM benchmark, which comprehensively measures visual detail perception from $9$ fine-grained aspects. As shown in \cref{tab: understanding ability results}, our method achieves consistent improvements across $6$ representative CLIP backbones. Specifically, it surpasses the original model by $14.1\%$ on the OpenAI CLIP ViT-L@224 backbone and $8.9\%$ on the MetaCLIP ViT-L@224 backbone. Our approach achieves the best performance even when compared with GenHancer and un$^2$CLIP, which employ dedicated generative models during training. We further provide qualitative results in \cref{fig: understanding ability qualitative results}, demonstrating that our enhanced CLIP serves as a plug-and-play module to improve the fine-grained perception, such as color and perspective.

\vskip 0.5 ex
\noindent \textbf{Performance of D-Ability.} We assess class separability by zero-shot clustering on $6$ standard datasets. As shown in \cref{tab: discriminative ability results}, our method delivers consistent gains over all backbones and baselines. It achieves the best average NMI/ACC/ARI, with the strongest result on SigLIP ViT-SO@224 (avg. $0.83/0.76/0.65$). The improvements are pronounced on fine-grained or texture-biased sets such as Caltech-101 and DTD, while remaining robust on general-purpose sets like CIFAR-10 and ImageNet-1K. Furthermore, we provide t-SNE visualizations on the MNIST dataset using the OpenAI CLIP ViT-L@224 backbone, as shown in \cref{fig: discriminative ability qualitative results}. The results demonstrate that our enhanced CLIP exhibits stronger class separability.

\begin{table}[t!]
\centering
\renewcommand\arraystretch{1}
\caption{Ablation studies in terms of \textbf{\textcolor{PerColor}{P-Ability}} and \textbf{\textcolor{DisColor}{D-Ability}}.}
\begin{minipage}{0.48\linewidth}
\centering
\subcaption{Naive Method}
\resizebox{\linewidth}{!}{
\begin{tabular}{ccccc}
\toprule
\multirow{2}[4]{*}{\textbf{Method}} & \textbf{\textcolor{PerColor}{MMVP}} & \multicolumn{3}{c}{\textbf{\textcolor{DisColor}{Clustering}}} \\
\cmidrule{2-5}
& ACC & NMI & ACC & ARI \\
\midrule
Naive & 22.96  & 0.72  & 0.65  & 0.50  \\
Ours  & \textbf{33.30} & \textbf{0.76} & \textbf{0.67} & \textbf{0.54} \\
\bottomrule
\end{tabular}}
\label{tab: Naive Method}
\end{minipage}
\hfill
\begin{minipage}{0.48\linewidth}
\centering
\subcaption{Different Training Protocol}
\resizebox{\linewidth}{!}{
\begin{tabular}{ccccc}
\toprule
\multirow{2}[4]{*}{\textbf{Method}} & \textbf{\textcolor{PerColor}{MMVP}} & \multicolumn{3}{c}{\textbf{\textcolor{DisColor}{Clustering}}} \\
\cmidrule{2-5}
& ACC & NMI & ACC & ARI \\
\midrule
End-to-End & 25.93  & 0.73  & 0.65  & 0.51  \\
Two-Stage & \textbf{33.30} & \textbf{0.76} & \textbf{0.67} & \textbf{0.54} \\
\bottomrule
\end{tabular}}
\label{tab: Different Training Protocol}
\end{minipage}
\vspace{-10pt}
\label{tab: Ablation Studies}
\end{table}

\subsection{Multimodal Large Language Model Evaluation}
\label{subsec: Multimodal Large Language Model Evaluation}
We further examine whether the enhanced visual representations learned by our method can be effectively transferred to \textit{Multimodal Large Language Models (MLLMs)}. Specifically, we substitute the original CLIP vision encoder in the LLaVA-1.5~\cite{liu2024improved} framework with our improved CLIP while keeping all training configurations unchanged to ensure a fair comparison. The evaluation is performed on a series of vision-centric and general multimodal benchmarks, covering both vision-centric and general tasks. As shown in \cref{tab: mllm results}, the integration of our enhanced CLIP substantially boosts MLLM performance, particularly in benchmarks that require detailed visual understanding. These results demonstrate that improving the granularity of visual representations not only benefits standalone vision models but also strengthens the visual grounding and reasoning ability of multimodal systems.

\subsection{Ablation Studies}
\label{subsec: Ablation Studies}

\begin{figure}[t]
  \centering
   \includegraphics[width=0.8\linewidth]{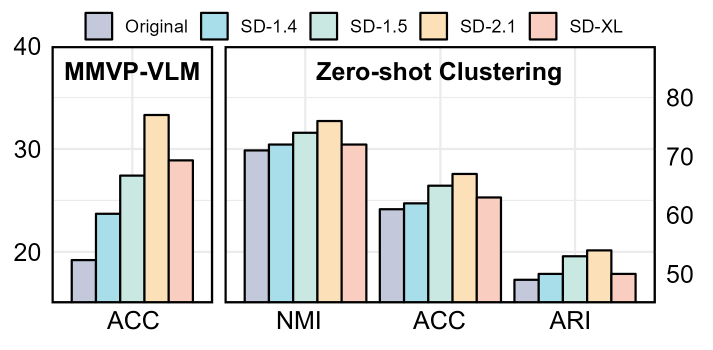}
   \vspace{-10pt}
   \caption{Ablation study on different diffusion model structures.}
   \vspace{-15pt}
   \label{fig: ablation_studies_sds}
\end{figure}

We perform several ablation studies to test the effectiveness of different modules and hyperparameters. All experiments are conducted on the OpenAI CLIP ViT-L@224.

\noindent \textbf{Results on the Naive Method.} \cref{tab: Naive Method} compares our DCR with the naive linear weighting method. The results show that the naive method improves the \textbf{\textcolor{DisColor}{D-Ability}} to some extent, but suffers from poor \textbf{\textcolor{PerColor}{P-Ability}} (only $22.96\%$) due to gradient conflicts. In contrast, our method employs a unified loss function to eliminate these conflicts and achieves better results in both aspects. This demonstrates the effectiveness of our proposed DCR framework.

\vskip 0.5 ex
\noindent \textbf{Results on Different Diffusion Model Structures.} \cref{fig: ablation_studies_sds} presents the performance of our method using different \textit{Stable Diffusion (SD)} backbones, including SD-1.4, SD-1.5, SD-2.1, and SD-XL. We observe a consistent trend that stronger diffusion models generally achieve better results in both \textbf{\textcolor{DisColor}{D-Ability}} and \textbf{\textcolor{PerColor}{P-Ability}} metrics. SD-2.1 attains the best overall performance. It is worth noting that although SD-XL is trained on larger and higher-quality data, it adopts a dual-encoder architecture for conditioning. When replaced with a single image-based embedding as the condition, its potential cannot be fully utilized, leading to a drop in final performance.

\vskip 0.5 ex
\noindent \textbf{Results on Different Training Protocol.} \cref{tab: Different Training Protocol} compares two training protocols: end-to-end and two-stage optimization. The two-stage protocol consistently outperforms the end-to-end scheme across all metrics. This suggests that pretraining the projector helps improve the reconstruction quality of the diffusion model, thereby facilitating better CLIP enhancement.
\section{Conclusion}
\label{sec: conclusion}
This paper presents \textit{Diffusion Contrastive Reconstruction (DCR)}, a general optimization framework that guides diffusion-based reconstruction with contrastive signals to balance discriminative and detail perceptual abilities. DCR effectively removes gradient conflicts and achieves consistent improvements in both aspects. Theoretical analysis and empirical results demonstrate that DCR not only enhances CLIP visual representations but also strengthens the visual reasoning capability of multimodal large language models.

\section*{Acknowledgments}
\label{sec: Acknowledgments}
This work was supported in part by the National Natural Science Foundation of China: 62525212, U23B2051, 62236008, 62441232, 62521007, U21B2038, 62502496 and 62576332, in part by the Youth Innovation Promotion Association CAS, in part by the Strategic Priority Research Program of the Chinese Academy of Sciences, Grant No. XDB0680201, in part by the Beijing Major Science and Technology Project under Contract No. Z251100008125059, in part by Beijing Academy of Artificial Intelligence (BAAI), in part by the China Postdoctoral Science Foundation (CPSF) under Grant No. 2025M771492, in part by the Postdoctoral Fellowship Program of CPSF under Grant No. GZB20240729, in part by the Young Elite Scientists Sponsorship Program of the Beijing High Innovation Plan, and in part by the project of Shandong Provincial Natural Science Foundation ZR2025ZD01.

{
    \small
    \bibliographystyle{ieeenat_fullname}
    \bibliography{main}
}

\newpage
\appendix
\setcounter{page}{1}
\maketitlesupplementary

{
\definecolor{blue}{RGB}{0,20,115}
\section*{\textcolor{blue}{\Large{Contents}}}
\startcontents[appendices]
\printcontents[appendices]{l}{1}{\setcounter{tocdepth}{3}}
}

\section{Symbol Definitions}
\label{sec: supp_Symbol Definitions}

In this section, \Cref{tab: symbols} includes a summary of key notations and descriptions in this work.

\begin{table}[htbp]
    \centering
    \renewcommand{\arraystretch}{1.2}
    \caption{A summary of key notations and descriptions in this work.}
    \resizebox{\linewidth}{!}{
    \begin{tabular}{ll}
        \toprule \textbf{Notations}                                  & \textbf{Descriptions}                                                \\
        \midrule $\x \in \mathbb{R}^{H \times W \times 3}$                                             & Input image.                         \\
        $\mathcal{D} = \{\x_i\}_{i=1}^n$                                                      & Training dataset.                         \\
        $\tilde{\x}$                                     & VAE latent of image $\x$.                  \\
        $t \in \{1,\dots,T\}$                                     & Diffusion timestep.     \\
        $\x_t$                                                 & Noisy latent at diffusion step $t$. \\
        $\epsilon_{\theta}$                                        & Diffusion noise prediction network.                   \\
        $f_{\phi}$                                                      & CLIP visual encoder.              \\
        $\z$                                                      & Visual feature extracted by CLIP encoder, shorthand for $\z = f_{\phi}(\x)$.                                               \\
        $h_{\omega}$                                 & Projection module that maps CLIP features to diffusion condition space.                            \\
        $c$                     & Condition used for diffusion denoising, shorthand for $c = h_{\omega}(\z)$.       \\
        $\hat{\boldsymbol\epsilon}_{\theta}$                                                          & Predicted noise, shorthand for $\hat{\boldsymbol\epsilon}_{\theta} = \epsilon_{\theta}(\tilde{\x}_t, c, t)$.                                                       \\
        $\boldsymbol\epsilon_t^{\text{gt}}$                                                  & Ground-truth diffusion noise at step $t$.                                        \\
        $\hat{\boldsymbol\epsilon}_{+}$/$\hat{\boldsymbol\epsilon}_{-}$                                                          & Predicted noise of positive/negative sample.                                                       \\
        $P$                                                          & Positive set, shorthand for $P=\{\hat{\boldsymbol\epsilon}_{+},\boldsymbol\epsilon_{\text{gt}}\}$.                          \\
        $N$                                                          & Negative set, shorthand for $N=\{\hat{\boldsymbol\epsilon}_{-}^{k}\}_{k=1}^{N-1}$.                               \\
        $C$                                           & Set $C=P\cup N$.                              \\
        $\mathrm{sim}(u,v)$            & Cosine similarity $\mathrm{sim}(u,v)=\frac{\langle u,v\rangle}{|u|,|v|}$.             \\
        \bottomrule
    \end{tabular}%
    } \label{tab: symbols}
\end{table}

\newpage

\section{Algorithm of DCR}
\label{sec: supp_Algorithm of DCR}

For clarity, we provide the pseudocode of our \textit{Diffusion Contrastive Reconstruction (DCR)} training procedure in Algorithm \ref{alg: dcr}.

\begin{algorithm}[htbp]
   \caption{DCR Algorithm.}
   \label{alg: dcr}
   \begin{algorithmic}[1]
   \REQUIRE Dataset $\mathcal{D}=\{\x_i\}$; CLIP vision encoder $f_{\phi}$; projector $h_{\omega}$; frozen diffusion model $\epsilon_{\theta}$; VAE encoder $\mathrm{Enc}$; diffusion steps $T$; batch size $N_b$; temperature $\tau$; Stage-1 epochs $E_1$; Stage-2 epochs $E_2$;
   \ENSURE Enhanced CLIP encoder $f_{\phi}$;
   \STATE \textcolor{orange}{$\rhd$ Stage-1: Projector Alignment}
   \FOR{$epoch = 1$ \textbf{to} $E_1$}
      \STATE Sample mini-batch $\{\x_i\}_{i=1}^{N_b}$ from $\mathcal{D}$;
      \FOR{$i = 1$ \textbf{to} $N_b$}
         \STATE Sample timestep $t \sim \mathrm{Unif}(1,T)$ and noise $\boldsymbol\epsilon_{t}^{\text{gt}} \sim \mathcal{N}(0,I)$;
         \STATE $\tilde{\x}_i \gets \mathrm{Enc}(\x_i)$;
         \STATE $\x_{t} \gets \sqrt{\bar{\alpha}_{t}}\,\tilde{\x}_i + \sqrt{1-\bar{\alpha}_{t}}\,\boldsymbol\epsilon_{t}^{\text{gt}}$;

         \STATE $\hat{\boldsymbol\epsilon}_{i} \gets \epsilon_{\theta}(\x_{t}, h_{\omega}(f_{\phi}(\x_i)), t)$;

         \STATE $\x_i^{+} \gets a(\x_i)$;
         \STATE $\hat{\boldsymbol\epsilon}_{+,i} \gets \epsilon_{\theta}(\x_{t}, h_{\omega}(f_{\phi}(\x_i^{+})), t)$;

         \STATE $N_i \gets \emptyset$;
         \FOR{$j = 1$ \textbf{to} $N_b$, $j \neq i$}
             \STATE $\hat{\boldsymbol\epsilon}_{-,i}^{j} \gets \epsilon_{\theta}(\x_{t}, h_{\omega}(f_{\phi}(\x_j)), t)$;
             \STATE $N_i \gets N_i \cup \{\hat{\boldsymbol\epsilon}_{-,i}^{j}\}$;
         \ENDFOR

         \STATE $P_i \gets \{\hat{\boldsymbol\epsilon}_{+,i},\, \boldsymbol\epsilon_{t}^{\text{gt}}\}$;
         \STATE $C_i \gets P_i \cup N_i$;

         \STATE $\displaystyle
         \mathcal{L}_i \gets 
         -\frac{1}{2} \sum_{p \in P_i}
         \log\frac{\exp(\mathrm{sim}(\hat{\boldsymbol\epsilon}_{i},p)/\tau)}
                   {\sum_{c\in C_i}\exp(\mathrm{sim}(\hat{\boldsymbol\epsilon}_{i},c)/\tau)}$;
      \ENDFOR
      \STATE $\mathcal{L}_{dcr} \gets \frac{1}{N_b}\sum_{i=1}^{N_b}\mathcal{L}_i$;
      \STATE Update projector parameters $\omega$;
   \ENDFOR

   \STATE \textcolor{orange}{$\rhd$ Stage-2: Encoder Enhancement}
   \FOR{$epoch = 1$ \textbf{to} $E_2$}
      \STATE Sample mini-batch $\{\x_i\}_{i=1}^{N_b}$ and repeat the same steps to compute $P_i$, $N_i$, and $\mathcal{L}_i$;
      \STATE $\mathcal{L}_{dcr} \gets \frac{1}{N_b}\sum_{i=1}^{N_b}\mathcal{L}_i$;
      \STATE Update CLIP vision encoder $f_{\phi}$;
   \ENDFOR

   \STATE \textbf{return} $f_{\phi}$;
   \end{algorithmic}
\end{algorithm}

\newpage
\section{Proof of Theorem 1}
\label{sec: supp_Proof of Theorem 1}

\begin{defn}[Bi-Lipschitz Mapping~\cite{david2016bi}]
Let $T:\mathcal{X} \to \mathcal{Y}$ be a mapping between two metric spaces $(\mathcal{X}, \|\cdot\|_{\mathcal{X}})$ and $(\mathcal{Y}, \|\cdot\|_{\mathcal{Y}})$. We say that $T$ is bi-Lipschitz on a subset $\mathcal{M} \subseteq \mathcal{X}$ if there exist constants $0 < m \le L < \infty$ such that, for all $z_1, z_2 \in \mathcal{M}$,
\begin{equation}
m\|z_1 - z_2\|_{\mathcal{X}} \le \bigl\|T(z_1) - T(z_2)\bigr\|_{\mathcal{Y}} \le L \|z_1 - z_2\|_{\mathcal{X}}.
\end{equation}
In particular, a bi-Lipschitz map preserves pairwise distances up to constant factors and is injective on $\mathcal{M}$.
\label{defn: bilip}
\end{defn}

\begin{assmp}
Fix a diffusion step $t$ and define $T(\z)=\epsilon_{\theta}\bigl(x_t,\,h_{\omega}(\z),\,t\bigr)$. According to Definition \ref{defn: bilip}, we assume that $T(\z)$ satisfies the bi-Lipschitz property.
\label{assmp: bilip}
\end{assmp}

\begin{lem}[Equivalent Variance Identity~\cite{zhang2012some}]
Let $e_1,\dots,e_n \in \mathbb{R}^d$ be arbitrary vectors and denote their mean by $\bar{e} = \frac{1}{n}\sum_{i=1}^n e_i$. Then the empirical variance can be written as the average of all pairwise squared distances:
\begin{equation}
\sum_{i=1}^n \|e_i - \bar{e}\|_2^2 = \frac{1}{2n}\sum_{i=1}^n \sum_{j=1}^n \|e_i - e_j\|_2^2.
\label{eq: var_pair}
\end{equation}
\label{lem: Equivalent Variance Identity}
\end{lem}

\begin{proof}
\textbf{For the left-hand side}, using $\bar{e} = \frac{1}{n}\sum_{k=1}^n e_k$, we have:
\begin{equation}
\begin{aligned}
\sum_{i=1}^n \|e_i - \bar{e}\|_2^2
&= \sum_{i=1}^n \langle e_i - \bar{e}, e_i - \bar{e} \rangle \\
&= \sum_{i=1}^n \bigl(\|e_i\|_2^2 + \|\bar{e}\|_2^2 - 2\langle e_i, \bar{e} \rangle \bigr) \\
&= \sum_{i=1}^n \|e_i\|_2^2 + n\|\bar{e}\|_2^2 - 2\left\langle \sum_{i=1}^n e_i, \bar{e} \right\rangle.
\end{aligned}
\end{equation}

Since $\sum_{i=1}^n e_i = n\bar{e}$, it follows that
\begin{equation}
\left\langle \sum_{i=1}^n e_i, \bar{e} \right\rangle = \langle n\bar{e}, \bar{e} \rangle = n\|\bar{e}\|_2^2.
\end{equation}

Therefore,
\begin{equation}
\sum_{i=1}^n \|e_i - \bar{e}\|_2^2 = \sum_{i=1}^n \|e_i\|_2^2 - n\|\bar{e}\|_2^2.
\label{eq: var_left}
\end{equation}

\textbf{For the right-hand side}, using $\|e_i - e_j\|_2^2 = \|e_i\|_2^2 + \|e_j\|_2^2 - 2\langle e_i, e_j\rangle$, we obtain
\begin{equation}
\begin{aligned}
&\sum_{i=1}^n \sum_{j=1}^n \|e_i - e_j\|_2^2 \\
=& \sum_{i,j} \|e_i\|_2^2 + \sum_{i,j} \|e_j\|_2^2 - 2\sum_{i,j} \langle e_i, e_j\rangle \\
=& n \sum_{i=1}^n \|e_i\|_2^2 + n \sum_{j=1}^n \|e_j\|_2^2 - 2 \left\langle \sum_{i=1}^n e_i,\ \sum_{j=1}^n e_j \right\rangle \\
=& 2n \sum_{i=1}^n \|e_i\|_2^2 - 2 \left\|\sum_{i=1}^n e_i\right\|_2^2.
\end{aligned}
\end{equation}

Again using $\sum_{i=1}^n e_i = n\bar{e}$, we obtain
\begin{equation}
\left\|\sum_{i=1}^n e_i\right\|_2^2 = \|n\bar{e}\|_2^2 = n^2\|\bar{e}\|_2^2,
\end{equation}
so
\begin{equation}
\sum_{i=1}^n \sum_{j=1}^n \|e_i - e_j\|_2^2 = 2n \sum_{i=1}^n \|e_i\|_2^2 - 2n^2\|\bar{e}\|_2^2.
\label{eq: var_right}
\end{equation}

Dividing \cref{eq: var_right} by $2n$ yields
\begin{equation}
\frac{1}{2n} \sum_{i=1}^n \sum_{j=1}^n \|e_i - e_j\|_2^2
= \sum_{i=1}^n \|e_i\|_2^2 - n\|\bar{e}\|_2^2.
\label{eq: var_right_2n}
\end{equation}

Therefore, by combining \cref{eq: var_left} and \cref{eq: var_right_2n}, we can obtain:
\begin{equation}
\sum_{i=1}^n \|e_i - \bar{e}\|_2^2 = \frac{1}{2n}\sum_{i=1}^n \sum_{j=1}^n \|e_i - e_j\|_2^2.
\end{equation}

This completed the proof.
\end{proof}

\begin{rthm1}
Fix a diffusion step $t$ and define $T(\z)=\epsilon_{\theta}\bigl(x_t,\,h_{\omega}(\z),\,t\bigr)$. Under a mild assumption, the intra-class scatter and inter-class scatter in the feature space ($S_{\text{inner}}, S_{\text{inter}}$) can be bounded by those in the noise space ($S_{\text{inner}}^{(\epsilon)}(t), S_{\text{inter}}^{(\epsilon)}(t)$) as
\begin{equation*}
S_{\text{inner}}\le\frac{1}{m^2}\,S_{\text{inner}}^{(\epsilon)}(t),
\end{equation*}
\begin{equation*}
S_{\text{inter}}\ge\kappa S_{\text{inter}}^{(\epsilon)}(t)-\eta S_{\text{inner}}^{(\epsilon)}(t),
\end{equation*}
where $m$, $\kappa$, and $\eta$ are positive constants depending on the Lipschitz continuity of $T(\cdot)$.
\end{rthm1}

\begin{proof}
We proceed in three steps. Throughout, fix a diffusion step $t$ and a mini-batch with class set $\mathcal Y$.

For class $y\in\mathcal Y$, let the feature set be $\mathcal Z_y=\{\z_i\}_{i=1}^{n_y}$ with $\z_i=f_{\phi}(\x_i)$ and the noise set be $\mathcal E_{y,t}=\{\hat{\boldsymbol\epsilon}_i\}_{i=1}^{n_y}$ with $\hat{\boldsymbol\epsilon}_i=\epsilon_{\theta}(x_t,h_{\omega}(\z_i),t)$. Denote the class means by $\mu_y=\frac{1}{n_y}\sum_{i=1}^{n_y}\z_i$ and $\nu_{y,t}=\frac{1}{n_y}\sum_{i=1}^{n_y}\hat{\boldsymbol\epsilon}_i$. As shown in Eq. (4) and Eq. (5) in the main text, we define the intra-/inter-class scatters in the CLIP feature space by
\begin{equation}
S_{\text{inner}}=\frac{1}{|\mathcal Y|}\sum_{y\in\mathcal Y}\frac{1}{n_y}\sum_{i=1}^{n_y}\|\z_i-\mu_y\|_2^2,
\label{eq: feature_inner}
\end{equation}
\begin{equation}
S_{\text{inter}}=\frac{1}{|\mathcal Y|(|\mathcal Y|-1)}\sum_{\substack{y,y'\in\mathcal Y\\y\ne y'}}\|\mu_y-\mu_{y'}\|_2^2.
\label{eq: feature_inter}
\end{equation}
Similarly, we define the intra-/inter-class scatters in the diffusion noise space by
\begin{equation}
S_{\text{inner}}^{(\epsilon)}(t)=\frac{1}{|\mathcal Y|}\sum_{y\in\mathcal Y}\frac{1}{n_y}\sum_{i=1}^{n_y}\|\hat{\boldsymbol\epsilon}_i-\nu_{y,t}\|_2^2,
\label{eq: noise_inner}
\end{equation}
\begin{equation}
S_{\text{inter}}^{(\epsilon)}(t)=\frac{1}{|\mathcal Y|(|\mathcal Y|-1)}\sum_{\substack{y,y'\in\mathcal Y\\y\ne y'}}\|\nu_{y,t}-\nu_{y',t}\|_2^2.
\label{eq: noise_inter}
\end{equation}

\vskip 0.5 ex
\noindent \textbf{Step 1: Minimizing $\mathcal L_{\mathrm{DCR}}$ improves noise-space geometry.}

Recall the DCR loss
\begin{equation*}
\mathcal{L}_{\mathrm{DCR}}=-\frac{1}{2}\sum_{p\in P} \log \frac{d(\hat{\boldsymbol\epsilon},p)}{\sum_{c\in C}d(\hat{\boldsymbol\epsilon},c)},
\end{equation*}
where $\hat{\boldsymbol\epsilon}=\epsilon_{\theta}(x_t,\,h_{\omega}(\z),\,t)$ is the anchor, $\hat{\boldsymbol\epsilon}_{+}$ is the positive sample, $\boldsymbol\epsilon_{\mathrm{gt}}$ is the ground-truth, and $N=\{\hat{\boldsymbol\epsilon}_{-}^{(j)}\}$ are negatives from other images in the mini-batch. $P=\{\hat{\boldsymbol\epsilon}_{+},\,\boldsymbol\epsilon_{t}^{\mathrm{gt}}\}$, $C=P\cup N$. $d(u,v)=\exp(\mathrm{sim}(u,v)/\tau)$ with $\mathrm{sim}$ the cosine similarity and $\tau>0$. Let
\begin{equation}
p(q\mid \hat{\boldsymbol\epsilon})=\frac{\exp(\mathrm{sim}(\hat{\boldsymbol\epsilon},q)/\tau)}{\sum_{c\in C}\exp(\mathrm{sim}(\hat{\boldsymbol\epsilon},c)/\tau)}.
\end{equation}
Differentiating the loss with respect to the similarity term gives the following gradient expression:
\begin{equation}
\frac{\partial \mathcal L_{\mathrm{DCR}}}{\partial\,\mathrm{sim}(\hat{\boldsymbol\epsilon},q)}
=\begin{cases}
-\tfrac{1}{2\tau}\bigl(1-2p(q\mid\hat{\boldsymbol\epsilon})\bigr), & q\in P,\\
\tfrac{1}{\tau}\,p(q\mid\hat{\boldsymbol\epsilon}), & q\in N.
\end{cases}
\end{equation}

Thus, gradient descent on $\mathcal L_{\mathrm{DCR}}$ \textbf{increases} the anchor's cosine similarity with the two positives ($\hat{\boldsymbol\epsilon}_{+}$ and $\boldsymbol\epsilon_{\mathrm{gt}}$) and \textbf{decreases} it with all negatives. Since
\begin{equation}
\mathrm{sim}(u,v)=\langle \bar u,\bar v\rangle,\quad \text{where}\ \bar u=\frac{u}{\|u\|},\ \bar v=\frac{v}{\|v\|},
\end{equation}
we have
\begin{equation}
\|\bar u-\bar v\|_2^2 = \|\bar u\|_2^2+\|\bar v\|_2^2-2\langle \bar u,\bar v\rangle = 2(1-\mathrm{sim}(\bar u,\bar v)).
\end{equation}

Therefore, minimizing $\mathcal L_{\mathrm{DCR}}$ is equivalent to
\begin{equation}
\min_{\phi,\omega}\bigl\|\hat{\boldsymbol\epsilon}-\hat{\boldsymbol\epsilon}_{+}\bigr\|_2^2+\bigl\|\hat{\boldsymbol\epsilon}-\boldsymbol\epsilon_{\mathrm{gt}}\bigr\|_2^2,
\label{eq: min_pair}
\end{equation}
\begin{equation}
\max_{\phi,\omega}\bigl\|\hat{\boldsymbol\epsilon}-\hat{\boldsymbol\epsilon}_{-}^{(j)}\bigr\|_2^2,\ \forall j.
\end{equation}

Then, we pass from pairwise distances to scatter. Specifically, for each class $y\in\mathcal Y$, the key component in \cref{eq: noise_inner},
\begin{equation}
\frac{1}{n_y}\sum_{i=1}^{n_y}\|\hat{\boldsymbol\epsilon}_i-\nu_{y,t}\|_2^2
\end{equation}
is, by Lemma~\ref{lem: Equivalent Variance Identity}, exactly the average of all pairwise squared distances $\|\hat{\boldsymbol\epsilon}_i-\hat{\boldsymbol\epsilon}_j\|_2^2$ within that class (up to the constant factor $1/(2n_y^2)$). Minimizing the pairwise distances between the anchor and its positive/ground-truth terms in \cref{eq: min_pair} therefore decreases the average intra-class pairwise distance, and thus decreases $\frac{1}{n_y}\sum_{i=1}^{n_y}\|\hat{\boldsymbol\epsilon}_i-\nu_{y,t}\|_2^2$ for each $y$. Averaging over all classes yields
\begin{equation}
\min_{\phi,\omega}\;S_{\text{inner}}^{(\epsilon)}(t).
\label{eq: noise_inner_decrease}
\end{equation}

In contrast, the negative terms $\{\hat{\boldsymbol\epsilon}_{-}^{(j)}\}$ are drawn from other images in the mini-batch, which typically belong to different classes. Maximizing the anchor--negative distances $\|\hat{\boldsymbol\epsilon}-\hat{\boldsymbol\epsilon}_{-}^{(j)}\|_2^2$ increases the average pairwise distances between samples of different classes. This increase, when propagated to the distances between the corresponding class means $\{\nu_{y,t}\}_{y\in\mathcal Y}$, leads to an increase of the inter-class scatter $S_{\text{inter}}^{(\epsilon)}(t)$:
\begin{equation}
\max_{\phi,\omega}\;S_{\text{inter}}^{(\epsilon)}(t).
\label{eq: noise_inter_increase}
\end{equation}

Combining \cref{eq: noise_inner_decrease} and \cref{eq: noise_inter_increase}, we conclude that minimizing $\mathcal L_{\mathrm{DCR}}$ improves the geometry in the noise space by decreasing intra-class scatter and increasing inter-class scatter.

\vskip 0.5 ex
\noindent \textbf{Step 2: Bounding feature-space intra-class scatter.}

Fix a class $y\in\mathcal Y$, and write its feature set as $\mathcal Z_y=\{\z_i\}_{i=1}^{n_y}$ and its noise set as $\mathcal E_{y,t}=\{\hat{\boldsymbol\epsilon}_i\}_{i=1}^{n_y}$ with $\hat{\boldsymbol\epsilon}_i = T(\z_i)$.

By Lemma~\ref{lem: Equivalent Variance Identity}, the intra-class variance in the feature space can be written as
\begin{equation}
\frac{1}{n_y}\sum_{i=1}^{n_y}\|\z_i-\mu_y\|_2^2 = \frac{1}{2n_y^2}\sum_{i=1}^{n_y}\sum_{j=1}^{n_y}\|\z_i-\z_j\|_2^2.
\label{eq: feature_var_pair}
\end{equation}
Similarly, the intra-class variance in the noise space is
\begin{equation}
\frac{1}{n_y}\sum_{i=1}^{n_y}\|\hat{\boldsymbol\epsilon}_i-\nu_{y,t}\|_2^2
= \frac{1}{2n_y^2}\sum_{i=1}^{n_y}\sum_{j=1}^{n_y}
\|\hat{\boldsymbol\epsilon}_i-\hat{\boldsymbol\epsilon}_j\|_2^2.
\label{eq: noise_var_pair}
\end{equation}

By Assumption~\ref{assmp: bilip}, there exists $m>0$ such that, for all $\z_1,\z_2\in\mathcal M$,
\begin{equation}
m\|\z_1-\z_2\| \le \|T(\z_1)-T(\z_2)\|.
\end{equation}

Equivalently,
\begin{equation}
\|\z_1-\z_2\| \le \frac{1}{m}\|T(\z_1)-T(\z_2)\|.
\end{equation}

Applying this inequality to each pair $(\z_i,\z_j)$ gives
\begin{equation}
\|\z_i-\z_j\|^2 \le \frac{1}{m^2}\|T(\z_i)-T(\z_j)\|^2 = \frac{1}{m^2}\|\hat{\boldsymbol\epsilon}_i-\hat{\boldsymbol\epsilon}_j\|^2.
\label{eq: feature_noise}
\end{equation}

Substituting \cref{eq: feature_noise} into \cref{eq: feature_var_pair} and comparing with \cref{eq: noise_var_pair}, we obtain
\begin{equation}
\begin{aligned}
\frac{1}{n_y}\sum_{i=1}^{n_y}\|\z_i-\mu_y\|_2^2 &\le \frac{1}{2n_y^2}\sum_{i,j=1}^{n_y} \frac{1}{m^2}\|\hat{\boldsymbol\epsilon}_i-\hat{\boldsymbol\epsilon}_j\|_2^2 \\
&= \frac{1}{m^2}\cdot\frac{1}{n_y}\sum_{i=1}^{n_y}\|\hat{\boldsymbol\epsilon}_i-\nu_{y,t}\|_2^2.
\end{aligned}
\end{equation}

Averaging over all classes $y\in\mathcal Y$ yields
\begin{equation}
\begin{aligned}
S_{\text{inner}} &= \frac{1}{|\mathcal Y|}\sum_{y\in\mathcal Y} \frac{1}{n_y}\sum_{i=1}^{n_y}\|\z_i-\mu_y\|_2^2 \\
&\le \frac{1}{m^2}\cdot \frac{1}{|\mathcal Y|}\sum_{y\in\mathcal Y} \frac{1}{n_y}\sum_{i=1}^{n_y}\|\hat{\boldsymbol\epsilon}_i-\nu_{y,t}\|_2^2 \\
&= \frac{1}{m^2} S_{\text{inner}}^{(\epsilon)}(t),
\end{aligned}
\end{equation}
which proves the first inequality in Theorem 1.

\vskip 0.5 ex
\noindent \textbf{Step 3: Bounding feature-space inter-class scatter.}

By Assumption~\ref{assmp: bilip}, there exists $L>0$ such that, for all $\z_1,\z_2\in\mathcal M$,
\begin{equation}
\|T(\z_1)-T(\z_2)\| \le L\|\z_1-\z_2\|.
\end{equation}
Equivalently,
\begin{equation}
\|\z_1-\z_2\| \ge \frac{1}{L}\|T(\z_1)-T(\z_2)\|.
\label{eq: lower_bound_T}
\end{equation}

Fix two distinct classes $y,y'\in\mathcal Y$. Applying \cref{eq: lower_bound_T} with $\z_1=\mu_y$, $\z_2=\mu_{y'}$ yields
\begin{equation}
\|\mu_y-\mu_{y'}\| \ge \frac{1}{L}\|T(\mu_y)-T(\mu_{y'})\|.
\label{eq: mu_lower_bound}
\end{equation}

We next relate $T(\mu_y)$ to $\nu_{y,t}$. We have
\begin{equation}
\nu_{y,t} = \frac{1}{n_y}\sum_{i=1}^{n_y}\hat{\boldsymbol\epsilon}_i = \frac{1}{n_y}\sum_{i=1}^{n_y} T(\z_i).
\end{equation}
Then
\begin{equation}
\begin{aligned}
\|T(\mu_y)-\nu_{y,t}\|_2
&= \left\|T(\mu_y)-\frac{1}{n_y}\sum_{i=1}^{n_y}T(\z_i)\right\|_2 \\
&= \left\|\frac{1}{n_y}\sum_{i=1}^{n_y}\bigl(T(\mu_y)-T(\z_i)\bigr)\right\|_2 \\
&\le \frac{1}{n_y}\sum_{i=1}^{n_y}\|T(\mu_y)-T(\z_i)\|_2 \\
&\le \frac{L}{n_y}\sum_{i=1}^{n_y}\|\mu_y-\z_i\|_2,
\end{aligned}
\end{equation}
where we used the triangle inequality and then the Lipschitz property of $T$. By the Cauchy-Schwarz inequality,
\begin{equation}
\sum_{i=1}^{n_y}\|\mu_y-\z_i\|_2 \le \sqrt{n_y}\left(\sum_{i=1}^{n_y}\|\mu_y-\z_i\|_2^2\right)^{1/2}.
\end{equation}
Hence,
\begin{equation}
\|T(\mu_y)-\nu_{y,t}\|_2 \le \frac{L}{\sqrt{n_y}} \left(\sum_{i=1}^{n_y}\|\mu_y-\z_i\|_2^2\right)^{1/2}.
\label{eq: Tmu_nu_bound}
\end{equation}

 Similarly, for class $y'$, we have:
\begin{equation}
\|T(\mu_{y'})-\nu_{y',t}\|_2
\le \frac{L}{\sqrt{n_{y'}}}
\left(\sum_{i=1}^{n_{y'}}\|\mu_{y'}-\z_i\|_2^2\right)^{1/2}.
\label{eq: Tmuprime_nuprime_bound}
\end{equation}

Using the triangle inequality, we get
\begin{equation}
\begin{aligned}
&\|T(\mu_y)-T(\mu_{y'})\|_2 \\
\ge& \|\nu_{y,t}-\nu_{y',t}\|_2
   - \|T(\mu_y)-\nu_{y,t}\|_2 \\
   &\ \ \ \ \ \ \ \ \ \ \ \ \ \ \ \ \ \ \ \ \ \ - \|T(\mu_{y'})-\nu_{y',t}\|_2.
\end{aligned}
\end{equation}

Combining this with \cref{eq: Tmu_nu_bound,eq: Tmuprime_nuprime_bound}, we obtain
\begin{equation}
\begin{aligned}
&\|T(\mu_y)-T(\mu_{y'})\|_2 \\
\ge& \|\nu_{y,t}-\nu_{y',t}\|_2
 - \frac{L}{\sqrt{n_y}}
   \left(\sum_{i=1}^{n_y}\|\mu_y-\z_i\|_2^2\right)^{1/2} \\
 &\ \ \ \ \ \ \ \ \ \ \ \ \ \ \ \ \ \ \ \ \ - \frac{L}{\sqrt{n_{y'}}}
   \left(\sum_{i=1}^{n_{y'}}\|\mu_{y'}-\z_i\|_2^2\right)^{1/2}.
\end{aligned}
\label{eq: Tmu_diff_bound}
\end{equation}

Substituting \cref{eq: Tmu_diff_bound} into \cref{eq: mu_lower_bound}, we obtain
\begin{equation}
\begin{aligned}
\|\mu_y-\mu_{y'}\|_2
&\ge \frac{1}{L}\|\nu_{y,t}-\nu_{y',t}\|_2 \\
&\quad - \frac{1}{\sqrt{n_y}}
   \left(\sum_{i=1}^{n_y}\|\mu_y-\z_i\|_2^2\right)^{1/2}\\
&\quad - \frac{1}{\sqrt{n_{y'}}}
   \left(\sum_{i=1}^{n_{y'}}\|\mu_{y'}-\z_i\|_2^2\right)^{1/2}.
\end{aligned}
\end{equation}

Squaring both sides and using the inequality $(a-b-c)^2 \ge \tfrac{1}{2}a^2 - 2(b^2+c^2)$ for any $a,b,c\in\mathbb R$, we obtain
\begin{equation}
\begin{aligned}
\|\mu_y-\mu_{y'}\|_2^2
&\ge \frac{1}{2L^2}\|\nu_{y,t}-\nu_{y',t}\|_2^2 \\
&\quad - 2\left(\frac{1}{n_y}\sum_{i=1}^{n_y}\|\mu_y-\z_i\|_2^2\right) \\
&\quad       - 2\left(\frac{1}{n_{y'}}\sum_{i=1}^{n_{y'}}\|\mu_{y'}-\z_i\|_2^2\right).
\end{aligned}
\end{equation}

Averaging the above inequality over all ordered pairs $(y,y')$ with $y\neq y'$,
and recalling the definition of $S_{\text{inter}}$ and $S_{\text{inter}}^{(\epsilon)}(t)$, we obtain
\begin{equation}
\begin{aligned}
S_{\text{inter}}
&= \frac{1}{|\mathcal Y|(|\mathcal Y|-1)}
   \sum_{\substack{y,y'\in\mathcal Y\\y\ne y'}}
   \|\mu_y-\mu_{y'}\|_2^2 \\
&\ge \frac{1}{|\mathcal Y|(|\mathcal Y|-1)}
   \sum_{\substack{y,y'\in\mathcal Y\\y\ne y'}}
   \left[
      \frac{1}{2L^2}\|\nu_{y,t}-\nu_{y',t}\|_2^2 
      - 2 B_{y,y'}
   \right] \\
&= \frac{1}{2L^2} S_{\text{inter}}^{(\epsilon)}(t)
   - \frac{2}{|\mathcal Y|(|\mathcal Y|-1)}
     \sum_{\substack{y,y'\in\mathcal Y\\y\ne y'}}
     B_{y,y'},
\end{aligned}
\end{equation}
where $B_{y,y'} = \frac{1}{n_y}\sum_{i\in I_y}\|\mu_y-\z_i\|_2^2 + \frac{1}{n_{y'}}\sum_{i\in I_{y'}}\|\mu_{y'}-\z_i\|_2^2$.

For the second term, note that for each fixed $y$, the quantity $A_y = \frac{1}{n_y}\sum_{i\in I_y}\|\mu_y-\z_i\|_2^2$ appears exactly $(|\mathcal Y|-1)$ times as the first index and $(|\mathcal Y|-1)$ times as the second index when summing over all ordered pairs $(y,y')$ with $y\neq y'$. Hence
\begin{equation}
\sum_{\substack{y,y'\in\mathcal Y\\y\ne y'}} 
\left(A_y + A_{y'}\right)
= 2(|\mathcal Y|-1)\sum_{y\in\mathcal Y} A_y.
\end{equation}

Therefore,
\begin{equation}
\begin{aligned}
&\frac{2}{|\mathcal Y|(|\mathcal Y|-1)}
\sum_{\substack{y,y'\in\mathcal Y\\y\ne y'}}
\left(A_y + A_{y'}\right) \\
=& \frac{2}{|\mathcal Y|(|\mathcal Y|-1)}\,
   2(|\mathcal Y|-1)\sum_{y\in\mathcal Y} A_y \\
=& \frac{4}{|\mathcal Y|}\sum_{y\in\mathcal Y} A_y \\
 =& 4 S_{\text{inner}}.
\end{aligned}
\end{equation}

Substituting this back, we obtain the explicit bound
\begin{equation}
S_{\text{inter}}
\;\ge\; \frac{1}{2L^2} S_{\text{inter}}^{(\epsilon)}(t)
      \;-\; 4 S_{\text{inner}}.
\label{eq:Sinter_feature_noise}
\end{equation}

Finally, using the bound from Step 2, $S_{\text{inner}} \;\le\; \frac{1}{m^2} S_{\text{inner}}^{(\epsilon)}(t)$, we arrive at
\begin{equation}
S_{\text{inter}} \ge \frac{1}{2L^2} S_{\text{inter}}^{(\epsilon)}(t) - \frac{4}{m^2} S_{\text{inner}}^{(\epsilon)}(t).
\end{equation}
Thus, we have:
\begin{equation}
S_{\text{inter}}\ge\kappa S_{\text{inter}}^{(\epsilon)}(t)-\eta S_{\text{inner}}^{(\epsilon)}(t),
\end{equation}
where $\kappa = \frac{1}{2L^2}$ and $\eta = \frac{4}{m^2}$. It proves the second inequality in Theorem 1.

This completed the proof.
\end{proof}

\section{Proof of Theorem 2}
\label{sec: supp_Proof of Theorem 2}

\begin{rthm2}
If negatives in the mini-batch are well separated from the anchor (\textit{i.e.}, $\mathrm{sim}(\hat{\boldsymbol\epsilon},\hat{\boldsymbol\epsilon}^{j}_{-})\ll \mathrm{sim}(\hat{\boldsymbol\epsilon},\boldsymbol\epsilon_{\text{gt}})$ for all $j\neq i$), and predicted-noise norms are bounded away from $0$ and $\infty$, then the DCR loss reduces to a scaled reconstruction loss up to an additive constant:
\begin{equation*}
\mathcal{L}_{\mathrm{dcr}}=\lambda\big\|\epsilon_{\theta}(\x_t,\ h_{\omega}(f_{\phi}(\tilde{\x})),t)-\boldsymbol\epsilon_t^{\text{gt}}\big\|_2^2 + c,
\end{equation*}
where $\lambda>0$, and $c$ is a constant.
\end{rthm2}

\begin{proof}
We consider a single anchor sample and drop the sample index for brevity.

Recall that the anchor is $\hat{\boldsymbol\epsilon} =\epsilon_{\theta}(x_t,\,h_{\omega}(f_{\phi}(\tilde{\x})),\,t)$, the two positives are $\hat{\boldsymbol\epsilon}_{+}$ and $\boldsymbol\epsilon_{\text{gt}}$, and $N=\{\hat{\boldsymbol\epsilon}^{j}_{-}\}_{j}$ denotes the set of negatives in the mini-batch. Let $P=\{\hat{\boldsymbol\epsilon}_{+},\boldsymbol\epsilon_{\text{gt}}\}$, $C=P\cup N$, and $\mathrm{sim}(u,v) = \frac{\langle u,v\rangle}{\|u\|_2\,\|v\|_2}$, $d(u,v)= \exp\bigl(\mathrm{sim}(u,v)/\tau\bigr)$, with $\tau>0$ the temperature. The DCR loss for this anchor is
\begin{equation}
\mathcal{L}_{\mathrm{DCR}} = -\frac{1}{2}\sum_{p\in P} \log\frac{d(\hat{\boldsymbol\epsilon},p)}{\sum_{c\in C} d(\hat{\boldsymbol\epsilon},c)}.
\end{equation}

Denote $u = \mathrm{sim}(\hat{\boldsymbol\epsilon},\boldsymbol\epsilon_{\text{gt}})$, $v = \mathrm{sim}(\hat{\boldsymbol\epsilon},\hat{\boldsymbol\epsilon}_{+})$, and write
\begin{equation}
Z_{\text{pos}}
= \sum_{p\in P}\exp\bigl(\mathrm{sim}(\hat{\boldsymbol\epsilon},p)/\tau\bigr)
= e^{u/\tau} + e^{v/\tau},
\end{equation}
\begin{equation}
Z_{\text{neg}}
= \sum_{j}\exp\bigl(\mathrm{sim}(\hat{\boldsymbol\epsilon},\hat{\boldsymbol\epsilon}^{j}_{-})/\tau\bigr).
\end{equation}

Then we can expand
\begin{equation}
\begin{aligned}
\mathcal{L}_{\mathrm{DCR}}
&= -\frac{1}{2}\sum_{p\in P}
   \left(
      \frac{\mathrm{sim}(\hat{\boldsymbol\epsilon},p)}{\tau}
      - \log\bigl(Z_{\text{pos}}+Z_{\text{neg}}\bigr)
   \right) \\
&= -\frac{1}{2\tau}(u+v)
   + \log\bigl(Z_{\text{pos}}+Z_{\text{neg}}\bigr).
\end{aligned}
\label{eq: Ldcr_uv}
\end{equation}

Using the notation in \cref{eq: Ldcr_uv}, we now show that $\mathcal{L}_{\mathrm{DCR}}$ is (up to scaling and an additive constant) equivalent to the reconstruction loss.

First, by the assumption that negatives are well separated from the anchor, there exist constants $\Delta>0$ and $B\in\mathbb{N}$ such that for all anchors in the mini-batch,
\begin{equation}
\mathrm{sim}(\hat{\boldsymbol\epsilon},\hat{\boldsymbol\epsilon}^{j}_{-}) \le u - \Delta,\quad \forall j, \qquad |N|\le B,
\end{equation}
where $u=\mathrm{sim}(\hat{\boldsymbol\epsilon},\boldsymbol\epsilon_{\text{gt}})$. Hence
\begin{equation}
Z_{\text{neg}} = \sum_{j}\exp\bigl(\mathrm{sim}(\hat{\boldsymbol\epsilon},\hat{\boldsymbol\epsilon}^{j}_{-})/\tau\bigr) \le B \exp\bigl((u-\Delta)/\tau\bigr).
\end{equation}

Since $Z_{\text{pos}}\ge e^{u/\tau}$, we get
\begin{equation}
\frac{Z_{\text{neg}}}{Z_{\text{pos}}} \le B e^{-\Delta/\tau} = \delta,
\end{equation}
where $\delta>0$ is a constant independent of the trainable parameters. Therefore
\begin{equation}
\begin{aligned}
&\log\bigl(Z_{\text{pos}}+Z_{\text{neg}}\bigr)\\
=& \log\Bigl(Z_{\text{pos}}\bigl(1+Z_{\text{neg}}/Z_{\text{pos}}\bigr)\Bigr)\\
=& \log Z_{\text{pos}} + \log\bigl(1+Z_{\text{neg}}/Z_{\text{pos}}\bigr),
\end{aligned}
\end{equation}
and, using $0\le Z_{\text{neg}}/Z_{\text{pos}}\le \delta$,
\begin{equation}
0 \le \log\bigl(1+Z_{\text{neg}}/Z_{\text{pos}}\bigr) \le \log(1+\delta) = C_{\text{neg}}.
\end{equation}

Thus, from \cref{eq: Ldcr_uv},
\begin{equation}
\mathcal{L}_{\mathrm{DCR}} = -\frac{1}{2\tau}(u+v) + \log Z_{\text{pos}} + \Delta_{\text{neg}},
\label{eq: Ldcr_pos}
\end{equation}
where $0 \le \Delta_{\text{neg}} \le C_{\text{neg}}$. It remains to analyze the positive part:
\begin{equation}
\ell_{\text{pos}}(u,v) = -\frac{1}{2\tau}(u+v) + \log\bigl(e^{u/\tau}+e^{v/\tau}\bigr),
\end{equation}
since
\begin{equation}
\mathcal{L}_{\mathrm{DCR}} = \ell_{\text{pos}}(u,v) + \Delta_{\text{neg}}.
\label{eq: Ldcr_split}
\end{equation}

Next, we express $u$ and $v$ in terms of squared Euclidean distances. Define the normalized vectors
\begin{equation}
\hat{e} = \frac{\hat{\boldsymbol\epsilon}}{\|\hat{\boldsymbol\epsilon}\|_2},
\qquad
g = \frac{\boldsymbol\epsilon_{\text{gt}}}{\|\boldsymbol\epsilon_{\text{gt}}\|_2},
\qquad
p = \frac{\hat{\boldsymbol\epsilon}_{+}}{\|\hat{\boldsymbol\epsilon}_{+}\|_2}.
\end{equation}

Then $u = \langle \hat{e},g\rangle$ and $v = \langle \hat{e},p\rangle$. For unit vectors $a,b$ we have
\begin{equation}
\|a-b\|_2^2 = 2(1-\langle a,b\rangle),
\end{equation}
so, defining
\begin{equation}
d_{\text{gt}}^2 = \|\hat{e}-g\|_2^2,
\qquad
d_{+}^2 = \|\hat{e}-p\|_2^2,
\end{equation}
we obtain
\begin{equation}
u = 1 - \frac{1}{2}d_{\text{gt}}^2,
\qquad
v = 1 - \frac{1}{2}d_{+}^2.
\end{equation}

Substituting into $\ell_{\text{pos}}$ gives
\begin{equation}
\ell_{\text{pos}} = \frac{1}{4\tau}\bigl(d_{\text{gt}}^2 + d_{+}^2\bigr) + \log\left(e^{-d_{\text{gt}}^2/(2\tau)} + e^{-d_{+}^2/(2\tau)}\right).
\label{eq: ellpos_d}
\end{equation}

Since $\hat{e},g,p$ are unit vectors, both squared distances are bounded:
\begin{equation}
0 \le d_{\text{gt}}^2 \le 4, \qquad 0 \le d_{+}^2 \le 4.
\label{eq: d_bounds}
\end{equation}

The logarithmic term in \cref{eq: ellpos_d} satisfies a bound. Let
\begin{equation}
a = -\frac{d_{\text{gt}}^2}{2\tau},
\qquad
b = -\frac{d_{+}^2}{2\tau},
\end{equation}
so that
\begin{equation}
\log\left(e^{-d_{\text{gt}}^2/(2\tau)} + e^{-d_{+}^2/(2\tau)}\right) = \log(e^a+e^b).
\end{equation}

For any real $a,b$, we have
\begin{equation}
\max\{a,b\} \le \log(e^a+e^b) \le \max\{a,b\} + \log 2.
\end{equation}

Using $0\le d_{\text{gt}}^2,d_{+}^2\le 4$ and $\tau>0$, we obtain constants $C_{\min},C_{\max}\in\mathbb R$ depending only on $\tau$ such that
\begin{equation}
C_{\min} \le \log\left(e^{-d_{\text{gt}}^2/(2\tau)} + e^{-d_{+}^2/(2\tau)}\right) \le C_{\max}
\label{eq: log_bound}
\end{equation}
for all possible $d_{\text{gt}}^2,d_{+}^2$. Moreover, the term $\frac{1}{4\tau}d_{+}^2$ in \cref{eq: ellpos_d} is also bounded using \cref{eq: d_bounds}, so it can be absorbed into the constants. Consequently, there exist finite constants $\tilde{C}_{\min},\tilde{C}_{\max}$ (depending only on $\tau$) such that
\begin{equation}
\frac{1}{4\tau}d_{\text{gt}}^2 + \tilde{C}_{\min} \le \ell_{\text{pos}} \le \frac{1}{4\tau}d_{\text{gt}}^2 + \tilde{C}_{\max}.
\label{eq: ellpos_bound_dgt}
\end{equation}

We now relate $d_{\text{gt}}^2$ to the true reconstruction error $\|\hat{\boldsymbol\epsilon}-\boldsymbol\epsilon_{\text{gt}}\|_2^2$. For any nonzero vectors $u,v$, the law of cosines gives
\begin{equation}
\|u-v\|_2^2 = \|u\|_2^2 + \|v\|_2^2 - 2\|u\|_2\|v\|_2\,\mathrm{sim}(u,v).
\end{equation}

With $u=\hat{\boldsymbol\epsilon}$, $v=\boldsymbol\epsilon_{\text{gt}}$ and $\mathrm{sim}(u,v)=\langle\hat{e},g\rangle=1-d_{\text{gt}}^2/2$, we obtain
\begin{equation}
\begin{aligned}
\bigl\|\hat{\boldsymbol\epsilon}-\boldsymbol\epsilon_{\text{gt}}\bigr\|_2^2
&= \|\hat{\boldsymbol\epsilon}\|_2^2 + \|\boldsymbol\epsilon_{\text{gt}}\|_2^2 - 2\|\hat{\boldsymbol\epsilon}\|_2\|\boldsymbol\epsilon_{\text{gt}}\|_2 \left(1-\frac{d_{\text{gt}}^2}{2}\right) \\
&= \bigl(\|\hat{\boldsymbol\epsilon}\|_2 - \|\boldsymbol\epsilon_{\text{gt}}\|_2\bigr)^2 + \|\hat{\boldsymbol\epsilon}\|_2\|\boldsymbol\epsilon_{\text{gt}}\|_2\,d_{\text{gt}}^2.
\end{aligned}
\label{eq: MSE_exact}
\end{equation}

By assumption, predicted-noise norms are bounded away from $0$ and $\infty$. Thus, there exist constants $0<\alpha\le\beta<\infty$ such that
\begin{equation}
\alpha \le \|\hat{\boldsymbol\epsilon}\|_2,\ \|\boldsymbol\epsilon_{\text{gt}}\|_2 \le \beta
\quad\text{for all anchors.}
\end{equation}

Using \cref{eq: MSE_exact} and the fact that $\bigl(\|\hat{\boldsymbol\epsilon}\|_2 - \|\boldsymbol\epsilon_{\text{gt}}\|_2\bigr)^2\ge 0$, we have
\begin{equation}
\bigl\|\hat{\boldsymbol\epsilon}-\boldsymbol\epsilon_{\text{gt}}\bigr\|_2^2 \ge \|\hat{\boldsymbol\epsilon}\|_2\|\boldsymbol\epsilon_{\text{gt}}\|_2\,d_{\text{gt}}^2 \ge \alpha^2 d_{\text{gt}}^2.
\end{equation}

On the other hand, using $\|\hat{\boldsymbol\epsilon}\|_2^2 + \|\boldsymbol\epsilon_{\text{gt}}\|_2^2 \le 2\beta^2$ and $\|\hat{\boldsymbol\epsilon}\|_2\|\boldsymbol\epsilon_{\text{gt}}\|_2 \le \beta^2$, we obtain
\begin{equation}
\bigl\|\hat{\boldsymbol\epsilon}-\boldsymbol\epsilon_{\text{gt}}\bigr\|_2^2 \le 2\beta^2 + \beta^2 d_{\text{gt}}^2.
\end{equation}

Rearranging these inequalities yields
\begin{equation}
\frac{1}{\beta^2} \left(\bigl\|\hat{\boldsymbol\epsilon}-\boldsymbol\epsilon_{\text{gt}}\bigr\|_2^2-2\beta^2\right) \le d_{\text{gt}}^2 \le \frac{1}{\alpha^2} \bigl\|\hat{\boldsymbol\epsilon}-\boldsymbol\epsilon_{\text{gt}}\bigr\|_2^2.
\label{eq: dgt_MSE}
\end{equation}

Combining \cref{eq: ellpos_bound_dgt} and \cref{eq: dgt_MSE}, we get affine
bounds on $\ell_{\text{pos}}$ in terms of the reconstruction error.
For the lower bound,
\begin{equation}
\begin{aligned}
\ell_{\text{pos}}
&\ge \frac{1}{4\tau}d_{\text{gt}}^2 + \tilde{C}_{\min} \\
&\ge \frac{1}{4\tau\beta^2} \left(\bigl\|\hat{\boldsymbol\epsilon} -\boldsymbol\epsilon_{\text{gt}}\bigr\|_2^2-2\beta^2\right) + \tilde{C}_{\min} \\
&= \frac{1}{4\tau\beta^2} \bigl\|\hat{\boldsymbol\epsilon}-\boldsymbol\epsilon_{\text{gt}}\bigr\|_2^2 + \left(\tilde{C}_{\min} - \frac{1}{2\tau}\right).
\end{aligned}
\end{equation}

For the upper bound,
\begin{equation}
\begin{aligned}
\ell_{\text{pos}} &\le \frac{1}{4\tau}d_{\text{gt}}^2 + \tilde{C}_{\max} \\
&\le \frac{1}{4\tau\alpha^2} \bigl\|\hat{\boldsymbol\epsilon}-\boldsymbol\epsilon_{\text{gt}}\bigr\|_2^2 + \tilde{C}_{\max}.
\end{aligned}
\end{equation}

Define
\begin{equation}
\lambda_{\min} = \frac{1}{4\tau\beta^2},
\qquad
\lambda_{\max} = \frac{1}{4\tau\alpha^2},
\end{equation}
and constants
\begin{equation}
c_{\min} = \tilde{C}_{\min} - \frac{1}{2\tau},
\qquad
c_{\max} = \tilde{C}_{\max}.
\end{equation}

Then
\begin{equation}
\lambda_{\min} \bigl\|\hat{\boldsymbol\epsilon}-\boldsymbol\epsilon_{\text{gt}}\bigr\|_2^2 + c_{\min} \le \ell_{\text{pos}} \le \lambda_{\max} \bigl\|\hat{\boldsymbol\epsilon}-\boldsymbol\epsilon_{\text{gt}}\bigr\|_2^2 + c_{\max}.
\label{eq: ellpos_affine}
\end{equation}

Finally, recalling \cref{eq: Ldcr_split} and the bound $0\le \Delta_{\text{neg}}\le C_{\text{neg}}$, we conclude that
\begin{equation}
\lambda_{\min} \bigl\|\hat{\boldsymbol\epsilon}-\boldsymbol\epsilon_{\text{gt}}\bigr\|_2^2 + c_{\min} \le \mathcal{L}_{\mathrm{DCR}} \le \lambda_{\max}
\bigl\|\hat{\boldsymbol\epsilon}-\boldsymbol\epsilon_{\text{gt}}\bigr\|_2^2 + c_{\max} + C_{\text{neg}}.
\end{equation}

Thus, $\mathcal{L}_{\mathrm{DCR}}$ is sandwiched between two affine functions of the reconstruction loss $\bigl\|\epsilon_{\theta}(x_t,h_{\omega}(f_{\phi}(\tilde{\x})),t) -\boldsymbol\epsilon_t^{\text{gt}}\bigr\|_2^2$ with strictly positive slopes. In other words, minimizing $\mathcal{L}_{\mathrm{DCR}}$ is equivalent (up to a positive scaling factor and an additive constant) to minimizing the standard reconstruction loss. We can summarize this equivalence in the form
\begin{equation}
\mathcal{L}_{\mathrm{DCR}} = \lambda\bigl\|\epsilon_{\theta}(x_t,h_{\omega}(f_{\phi}(\tilde{\x})),t) -\boldsymbol\epsilon_t^{\text{gt}}\bigr\|_2^2 + c,
\end{equation}
for some $\lambda\in[\lambda_{\min},\lambda_{\max}]$ and constant $c$.

This completed the proof.
\end{proof}

\newpage
\section{Additional Experimental Settings}
\label{sec: supp_Additional Experiment Settings}
In this section, we make a supplementation to Sec 5.1.

\subsection{CLIP Backbones}
\label{subsec: supp_CLIP Backbones}
Here, we provide a detailed summary of the $6$ types of CLIP backbones used in our experiments. They fall into $3$ categories:
\begin{itemize}
    \item \textbf{OpenAI CLIP ViT-L@224 and OpenAI CLIP ViT-L@336}~\cite{radford2021learning} establish the canonical formulation of contrastive language-image pretraining, defining the feature space that later CLIP variants build upon. Their large-scale but noisy training data gives them strong open-world generalization, while their ability to perceive fine-grained visual details remains limited. The 336-resolution variant uniquely extends CLIP's receptive field without altering the architecture, making it a standard choice for studying high-resolution alignment.
    \item \textbf{MetaCLIP ViT-L@224 and MetaCLIP ViT-H@224}~\cite{xudemystifying} distinguishes itself by reconstructing a high-quality, well-aligned training corpus using a principled filtering pipeline rather than relying on raw web data. This dataset-centric redesign leads to representations that are more stable, less noisy, and better calibrated than those of OpenAI CLIP. The ViT-H model further pushes scaling laws within the CLIP paradigm.
    \item \textbf{SigLIP ViT-SO@224 and SigLIP ViT-SO@384}~\cite{zhai2023sigmoid} departs from the classic softmax contrastive loss by introducing a sigmoid-based objective, fundamentally altering how positive and negative pairs influence training. This probabilistic formulation enables more fine-grained pairwise alignment and reduces overconfidence artifacts commonly observed in CLIP-like models. Its higher-resolution variant leverages the smoother objective to maintain stable training.
\end{itemize}

\subsection{Competitors}
\label{subsec: supp_Competitors}
Here we give a more detailed summary of the competitors mentioned in the experiments.

\begin{itemize}
    \item \textbf{Original CLIP}~\cite{radford2021learning,xudemystifying,zhai2023sigmoid} is the baseline vision encoder pretrained on large-scale image–text pairs. It provides generalization across diverse recognition and retrieval tasks. However, it lacks mechanisms for reconstructive feedback, which limits its ability to understand fine-grained visuals.
    \item \textbf{DIVA}~\cite{wangdiffusion} introduces diffusion-based visual feedback to refine CLIP features. It performs reconstruction conditioned on CLIP vision embeddings, enabling the model to recover more detailed visual information.
    \item \textbf{GenHancer}~\cite{ma2025genhancer} systematically investigates how generative models can enhance CLIP by refining conditioning design, denoising strategies, and generation paradigms. It shows that using global conditions and lightweight denoisers leads to more stable reconstruction-based representation learning. In addition, it extends the reconstruction process to discrete latent spaces.
    \item \textbf{un$^2$CLIP}~\cite{li20252} builds on the unCLIP framework by inverting the generative process so that the visual encoder can better capture fine-grained image details while remaining aligned with CLIP's original embedding space.
\end{itemize}

\subsection{Evaluation Protocol for P-Ability}
\label{subsec: supp_Evaluation Protocol for P-Ability}
\noindent \textbf{Datasets.} Following \cite{wangdiffusion,ma2025genhancer,li20252}, we use MMVP-VLM~\cite{tong2024eyes} to evaluate the \textbf{\textcolor{PerColor}{P-Ability}}. It is a benchmark designed to evaluate fine-grained visual perception by testing VLMs on a wide range of visual patterns. It contains human-designed symbols, shapes, transformations, and composites that isolate specific perceptual abilities such as symmetry, color matching, and geometric relations. The dataset includes tens of thousands of diverse pattern–label pairs, providing a controlled environment for assessing detailed visual understanding.

\vskip 0.5 ex
\noindent \textbf{Evaluation metrics.} Following \cite{tong2024eyes}, we report Accuracy (ACC) as the evaluation metric. ACC measures the proportion of samples for which the model correctly recognizes the underlying visual pattern, serving as a direct indicator of fine-grained visual perception.

\subsection{Evaluation Protocol for D-Ability}
\label{subsec: supp_Evaluation Protocol for D-Ability}
\noindent \textbf{Datasets.} Following \cite{radford2021learning}, we perform zero-shot clustering on $6$ standard datasets~\cite{lecun2002gradient,krizhevsky2009learning,helber2019eurosat,fei2004learning,cimpoi2014describing,deng2009imagenet} to evaluate the \textbf{\textcolor{DisColor}{D-Ability}}.
\begin{itemize}
    \item \textbf{MNIST}~\cite{lecun2002gradient} is a handwritten digit recognition dataset containing grayscale images of digits $0$-$9$. Each image is centered and normalized to $28 \times 28$ pixels, making it a standard benchmark for evaluating basic visual representations. It provides $60000$ training and $10000$ test samples.
    \item \textbf{CIFAR-10}~\cite{krizhevsky2009learning} is a natural image dataset designed for general object classification across $10$ categories. The images are low-resolution $32 \times 32$ color images that introduce significant appearance variation despite their small size. The dataset includes $50000$ training and $10000$ test images.
    \item \textbf{Eurosat}~\cite{helber2019eurosat} is a satellite image dataset for land-use and land-cover classification, containing $10$ classes such as agricultural areas, forests, and urban regions. Derived from Sentinel-2 satellite imagery, it includes $21600$ training and $5400$ test images, offering rich spatial and spectral diversity.
    \item \textbf{Caltech-101}~\cite{fei2004learning} designed for general object classification, includes $101$ object categories and a background class, featuring around $7650$ training and $3300$ test images. The images, collected at Caltech, present significant variation in scale, orientation, and lighting conditions.
    \item \textbf{Describable Textures Dataset (DTD)}~\cite{cimpoi2014describing} focuses on texture classification, featuring $47$ texture categories described using human-interpretable attributes. It offers $3760$ training and $1880$ test images, providing a unique challenge in recognizing visually distinctive patterns from natural and artificial sources.
    \item \textbf{ImageNet-1K}~\cite{deng2009imagenet} is a large-scale benchmark for visual recognition, covering $1000$ object categories spanning animals, scenes, and man-made objects. The images exhibit substantial diversity in viewpoint, background, resolution, and object appearance, making it a core testbed for evaluating high-capacity visual models. It contains about $1.28$ million training images and $50$ thousand validation images.
\end{itemize}

\vskip 0.5 ex
\noindent \textbf{Evaluation metrics.} Following \cite{li2024image}, we adopt $3$ widely-used metrics to evaluate the clustering performance.
\begin{itemize}
    \item \textbf{Normalized Mutual Information (NMI)} measures the mutual dependence between predicted clusters and ground-truth labels, normalized to $[0,1]$. Higher values indicate better alignment between the clustering structure and the true class partition.
    \item \textbf{Accuracy (ACC)} evaluates the best-matched assignment between predicted clusters and ground-truth classes. It computes the fraction of correctly assigned samples after optimal label permutation using the Hungarian algorithm.
    \item \textbf{Adjusted Rand Index (ARI)} quantifies the similarity between two partitions by counting pairwise agreements, adjusted for random chance. An ARI of $0$ corresponds to random clustering, while higher values indicate more consistent partitioning.
\end{itemize}

\subsection{Evaluation Protocol for MLLMs}
\label{subsec: supp_Evaluation Protocol for MLLMs}
Following \cite{tong2024eyes,ma2025genhancer,li20252}, we evaluate MLLMs using two complementary groups of benchmarks: \textbf{Vision-Centric Benchmarks}, which focus on fine-grained visual perception and scene understanding, and \textbf{Conventional MLLM Benchmarks}, which measure robustness, hallucination resistance, and multimodal reasoning. Together, they provide a comprehensive view of an MLLM's visual competence, reliability, and general multimodal capability.

\vskip 0.5 ex
\noindent \textbf{Vision-Centric Benchmarks.} We evaluate on $4$ datasets:
\begin{itemize}
    \item \textbf{MMVP-MLLM}~\cite{tong2024eyes} evaluates the model's fine-grained visual perception by testing its ability to identify structured visual patterns. We report classification accuracy (ACC).
    \item \textbf{NaturalBench}~\cite{li2024naturalbench} examines multimodal reasoning across natural images using four metrics: overall accuracy (Acc), question-type accuracy (Q-Acc), instance-level accuracy (I-Acc), and group-level accuracy (G-Acc).
    \item \textbf{CV-Bench 2D}~\cite{tong2024cambrian} measures perception quality on 2D dense prediction tasks. We report the accuracy (ACC) of ADE20K~\cite{zhou2017scene} and COCO~\cite{caesar2018coco}.
    \item \textbf{CV-Bench 3D}~\cite{tong2024cambrian} assesses 3D spatial understanding through structured 3D reasoning queries. We report overall accuracy (ACC).
\end{itemize}

\vskip 0.5 ex
\noindent \textbf{Conventional MLLM Benchmarks.} We evaluate on $3$ datasets:
\begin{itemize}
    \item \textbf{POPE}~\cite{li2023evaluating} evaluates object hallucination robustness under three settings, including random (rand), popularity-based (pop), and adversarial (adv). Higher scores indicate stronger resistance to hallucination.
    \item \textbf{SciQA-IMG}~\cite{lu2022learn} assesses scientific visual question answering involving diagrams, plots, and structured visual cues. We report accuracy (ACC).
    \item \textbf{Hallusion}~\cite{guan2024hallusionbench} measures multimodal robustness by testing whether the model avoids visually induced false inferences. We report the average accuracy across all query types.
\end{itemize}

\section{Additional Experimental Results}
\label{sec: supp_Additional Experiment Results}

\subsection{Expanded Version of Quantitative Results}
\label{subsec: supp_Expanded Version of Quantitative Results}
Here, we present an expanded version of the quantitative results.

\cref{tab: zero-shot classification and retrieval} evaluates the improved CLIP models on the two classical tasks, zero-shot image classification and zero-shot text and image retrieval. The competitor focuses only on fine-grained reconstruction, which neglects discriminative ability and leads to performance drops. In contrast, our method preserves and further improves discriminative ability, demonstrating its effectiveness.

\begin{table}[htbp]
\setlength\tabcolsep{2pt}
\centering
\renewcommand{\arraystretch}{1}
\caption{Performance on zero-shot classification and retrieval.}
\vspace{-10pt}
\resizebox{.98\linewidth}{!}{
\begin{tabular}{@{}lcccccccccc@{}}
\toprule
\multirow{2}{*}{Method} & \multicolumn{6}{c}{Classification} & \multicolumn{2}{c}{Retrieval-Image@5} & \multicolumn{2}{c}{Retrieval-Text@5} \\ \cmidrule(l){2-7} \cmidrule(l){8-9} \cmidrule(l){10-11} 
 & MNIST & C10 & Eur & C101 & DTD & IN-1K & Flickr30k & COCO & Flickr30k & COCO \\ \midrule
Original & 76.4 & \textbf{95.6} & 60.1 & 86.6 & \textbf{55.4} & 75.5 & 87.2 & \textbf{61.1} & \textbf{97.4} & 79.2 \\
Genhancer & 69.7 & 73.7 & 58.5 & 71.5 & 48.4 & 73.8 & 81.6 & 51.1 & 87.3 & 61.4 \\
\rowcolor[rgb]{ .851,  .882,  .957}Ours & \textbf{76.5} & \textbf{95.6} & \textbf{60.3} & \textbf{86.7} & \textbf{55.4} & \textbf{75.6} & \textbf{87.3} & \textbf{61.1} & 97.3 & \textbf{79.3} \\ \bottomrule
\end{tabular}
}
\label{tab: zero-shot classification and retrieval}
\end{table}

\cref{tab: different ratios of local tokens} reports the results under different ratios of local tokens (\texttt{[CLS]} + $n\%$ local tokens). We find that using too many local tokens leads to performance degradation. This may be because excessive local tokens provide overly strong local cues, making the reconstruction task too easy and thus weakening the supervision signal. This phenomenon is consistent with the findings in GenHancer~\cite{ma2025genhancer} and further supports the common behavior of reconstruction-based enhancement methods.

\begin{table}[htbp]
\centering
\renewcommand\arraystretch{0.8}
\caption{Results under different ratios of local tokens (\texttt{[CLS]} + $n\%$ local tokens).}
\resizebox{0.8\linewidth}{!}{
\begin{tabular}{ccccc}
\toprule
\multirow{2}[4]{*}{\textbf{Ratio}} & \textbf{\textcolor{PerColor}{MMVP-VLM}} & \multicolumn{3}{c}{\textbf{\textcolor{DisColor}{Clustering}}} \\
\cmidrule{2-5}
& ACC & NMI & ACC & ARI \\
\midrule
0\%  & \textbf{33.30} & \textbf{0.76} & \textbf{0.67} & \textbf{0.54} \\
10\% & 31.85  & \textbf{0.76}  & 0.63  & 0.53  \\
50\% & 23.70  & 0.73  & 0.60  & 0.46  \\
80\% & 21.48  & 0.69  & 0.57  & 0.44  \\
100\% & 20.74  & 0.65  & 0.55  & 0.42  \\
\bottomrule
\end{tabular}}
\vspace{-10pt}
\label{tab: different ratios of local tokens}
\end{table}

\subsection{Expanded Version of Qualitative Results}
\label{subsec: supp_Expanded Version of Qualitative Results}
Here, we present an expanded version of the qualitative results.

\cref{fig: expanded_mmvp_vlm} presents some qualitative examples from the MMVP-VLM benchmark. The results show that our method enhances fine-grained visual perception, leading to improved \textbf{\textcolor{PerColor}{P-Ability}}.

\cref{fig: expanded_mllm} shows several qualitative examples on the vision-centric benchmarks from MLLMs. The results demonstrate that our enhanced CLIP can be seamlessly integrated into MLLMs to improve their visual capabilities.

\begin{figure*}[htbp]
  \centering
   \includegraphics[width=0.95\linewidth]{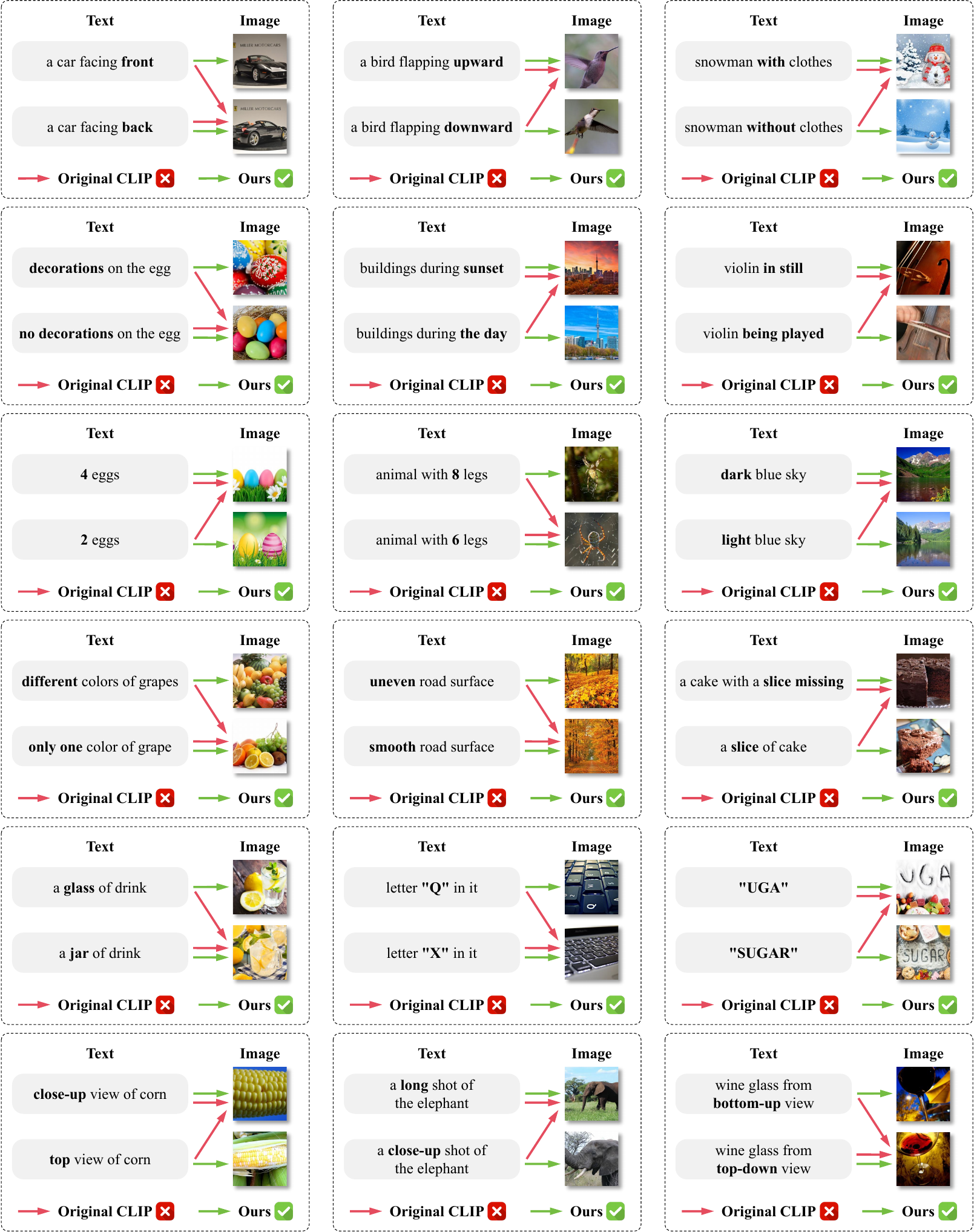}
   \caption{Expanded version of qualitative results on the MMVP-VLM benchmark. The predictions from the original CLIP and our improved version are indicated by red and green arrows, respectively. The improved CLIP effectively addresses the original model's limitations in capturing fine-grained visual details.}
   \label{fig: expanded_mmvp_vlm}
\end{figure*}

\begin{figure*}
  \centering
  \begin{subfigure}{0.9\linewidth}
    \includegraphics[width=\linewidth]{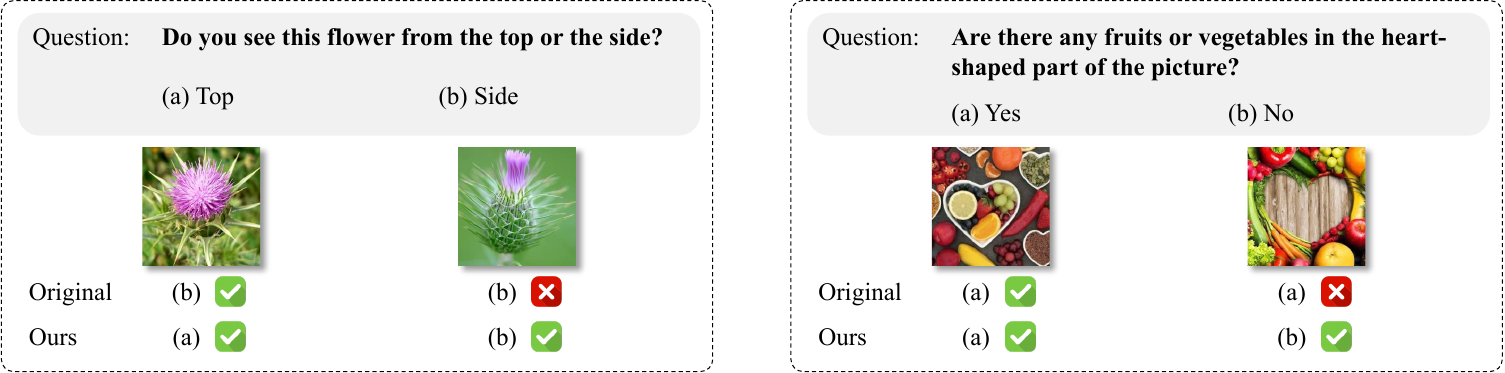}
    \caption{MMVP-MLLM.}
    \label{fig: expanded_mmvp_mllm}
  \end{subfigure}
  \hfill
  \begin{subfigure}{0.9\linewidth}
    \includegraphics[width=\linewidth]{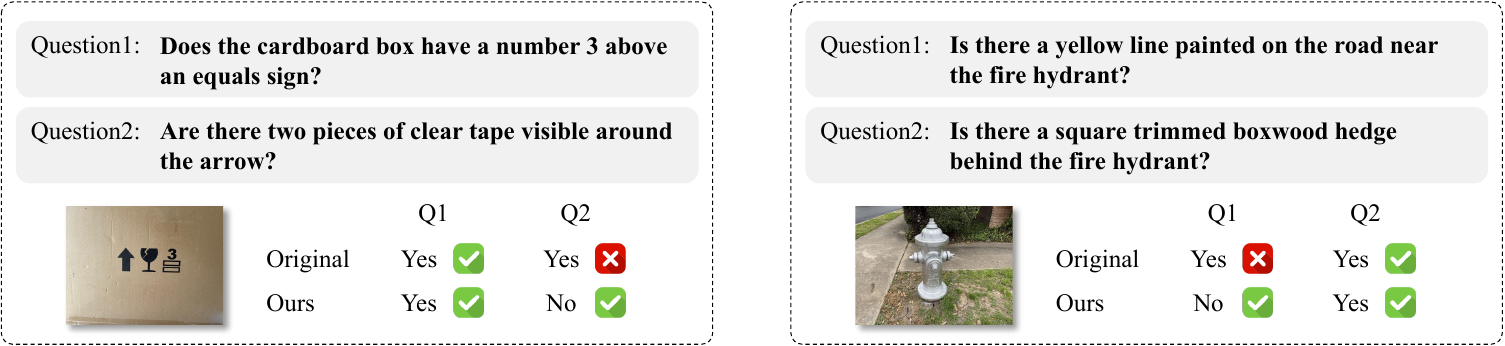}
    \caption{NaturalBench.}
    \label{fig: expanded_naturalbench}
  \end{subfigure}
  \caption{Expanded version of qualitative results on the MLLM benchmarks.}
  \label{fig: expanded_mllm}
\end{figure*}

\subsection{Computational Costs}
\label{subsec: supp_Computational Costs}
\cref{tab: Training costs} reports the computational costs of our method. Across different CLIP backbones, the training cost of DCR remains stable and mainly scales with the backbone size and input resolution. Models with similar parameter counts, such as OpenAI CLIP ViT-L@224, OpenAI CLIP ViT-L@336, and MetaCLIP ViT-L@224, exhibit almost identical training times, indicating that DCR introduces minimal overhead beyond the backbone’s forward passes. Larger or higher-resolution models, including MetaCLIP ViT-H@224 and SigLIP ViT-SO@384, incur predictable increases in training time and GPU memory usage, yet the overall computation remains well within a practical range for NVIDIA-A100 80GB GPU. These results show that DCR is computationally lightweight and can be seamlessly applied to a wide range of CLIP architectures without requiring specialized optimization.

\begin{table}[htbp]
  \centering
  \renewcommand\arraystretch{1}
  \caption{Training costs of our DCR.}
  \vspace{-10pt}
  \resizebox{0.9\linewidth}{!}{
    \begin{tabular}{cccc}
    \toprule
    \textbf{CLIP Backbone} &
    \textbf{\begin{tabular}[c]{@{}c@{}}Training\\ \#Params\end{tabular}} &
    \textbf{\begin{tabular}[c]{@{}c@{}}Training\\ Time\end{tabular}} &
    \textbf{\begin{tabular}[c]{@{}c@{}}Training\\ GPU Memory\end{tabular}} \\
    \midrule
    OpenAI CLIP ViT-L@224 & 350.5M & 7.7h  & 53544M \\
    OpenAI CLIP ViT-L@336 & 350.5M & 7.6h  & 60472M \\
    MetaCLIP ViT-L@224    & 350.5M & 7.7h  & 60472M \\
    MetaCLIP ViT-H@224    & 360.0M & 9.2h  & 76430M \\
    SigLIP ViT-SO@224     & 358.5M & 8.0h  & 66798M \\
    SigLIP ViT-SO@384     & 358.5M & 9.7h  & 78792M \\
    \bottomrule
    \end{tabular}
  }
\label{tab: Training costs}
\end{table}

\section{Future Works}
\label{sec: Future Works}
In the future, we plan to extend our diffusion contrastive reconstruction beyond diffusion models to VAR-based generators, enabling a broader family of generative priors to contribute fine-grained visual supervision. In parallel, we aim to investigate how our framework can be applied to strengthen visual encoders from diverse architectures, thereby providing a unified and adaptable enhancement strategy for a wide range of vision systems. Moreover, we believe that a systematic study of reconstruction-based methods is essential for uncovering the underlying principles of representation enhancement and for establishing rigorous theoretical guarantees.


\end{document}